\documentclass{article}

\usepackage{authblk}

\usepackage{times}
\usepackage{soul}
\usepackage{url}
\usepackage{hyperref}
\usepackage[utf8]{inputenc}
\usepackage[small]{caption}
\usepackage{graphicx}
\usepackage{amsmath}
\usepackage{amsthm}
\usepackage{booktabs}
\usepackage{algorithm}
\usepackage{algorithmic}
\usepackage{xspace}
\usepackage{enumerate}

\usepackage{amssymb}
\usepackage{stmaryrd}
\usepackage{todonotes}
\usepackage{paralist}
\usepackage{enumitem}

\newcommand{\wrt}{w.r.t.\ }
\newcommand{\size}[1]{\ensuremath{||#1||}\xspace}

\newtheorem{theorem}{Theorem} 
\newtheorem{definition}[theorem]{Definition}
\newtheorem{claim}[theorem]{Claim}
\newtheorem{observation}[theorem]{Observation}
\newtheorem{lemma}[theorem]{Lemma}

\newtheorem{proposition}[theorem]{Proposition}

\newtheorem{example}{Example}

\usetikzlibrary{decorations.pathmorphing}

\begin{document}

\title{An ExpTime Upper Bound for $\mathcal{ALC}$ with Integers (Extended Version)
}

\author{Nadia Labai}
\author{Magdalena Ortiz}
\author{Mantas \v Simkus}
\affil{Faculty of Informatics, Vienna, Austria}

\date{}

\maketitle        

\newif\ifskip
\skipfalse

\newcommand{\ALC}{\ensuremath{\mathcal{ALC}}\xspace}
\newcommand{\ALCF}{\mathcal{ALCF}}
\newcommand{\ALCP}{\mathcal{ALCFP}}
\newcommand{\ALCUP}{\mathcal{ALCF^P}}

\newcommand{\alcupz}{\ensuremath{\ALCUP(\cZ_c)}\xspace} 

\newcommand{\bfU}{\mathbf{U}}
\newcommand{\bfB}{\mathbf{B}}
\newcommand{\bfw}{\mathbf{w}}

\newcommand{\bN}{\mathbb{N}}
\newcommand{\bQ}{\mathbb{Q}}
\newcommand{\bZ}{\mathbb{Z}}

\newcommand{\cBv}{\mathcal{B}^{\wedge}}

\newcommand{\cA}{\mathcal{A}}
\newcommand{\cC}{\mathcal{C}}
\newcommand{\cB}{\mathcal{B}}
\newcommand{\cD}{\mathcal{D}}
\newcommand{\cG}{\mathcal{G}}
\newcommand{\cH}{\mathcal{H}}
\newcommand{\cI}{\mathcal{I}}
\newcommand{\cJ}{\mathcal{J}}
\newcommand{\cF}{\mathcal{F}}
\newcommand{\cK}{\mathcal{K}}
\newcommand{\cN}{\mathcal{N}}
\newcommand{\cO}{\mathcal{O}}
\newcommand{\cP}{\mathcal{P}}
\newcommand{\cQ}{\mathcal{Q}}
\newcommand{\cR}{\mathcal{R}}
\newcommand{\cT}{\mathcal{T}}
\newcommand{\cZ}{\ensuremath{\mathcal{Z}}\xspace}

\newcommand{\fr}{\mathrm{fr}}
\newcommand{\emb}{\mathrm{emb}}
\newcommand{\rtop}{\mathrm{top}}
\newcommand{\rbot}{\mathrm{bot}}

\newcommand{\rS}{\mathrm{S}}
\newcommand{\rct}{\mathrm{ct}}
\newcommand{\rCFr}{\mathrm{CFr}}
\newcommand{\ralcf}{\mathrm{alcf}}

\newcommand{\sfnat}{\mathsf{nat}}
\newcommand{\sfint}{\mathsf{int}}
\newcommand{\sfNC}{\mathsf{N_C}}
\newcommand{\sfNR}{\mathsf{N_R}}
\newcommand{\sfNF}{\mathsf{N_F}}
\newcommand{\sfNaf}{\mathsf{N_{af}}}
\newcommand{\sfP}{\mathsf{P}}
\newcommand{\sfReg}{\mathsf{Reg}}
\newcommand{\sfROL}{\mathsf{ROL}}
\newcommand{\sfund}{\mathsf{und}}
\newcommand{\sfpropagate}{\mathsf{propagate}}
\newcommand{\sfu}{\mathsf{u}}

\newcommand{\twoEXPTIME}{\textsc{2ExpTime}}
\newcommand{\EXPTIME}{\textsc{ExpTime}}
\newcommand{\NEXPTIME}{\textsc{NExpTime}}
\newcommand{\PTIME}{\textsc{PTime}}
\newcommand{\PSPACE}{\textsc{PSpace}}
\newcommand{\lpath}{\llbracket}
\newcommand{\rpath}{\rrbracket}

\newcommand{\condsym}{\bigstar}

\begin{abstract}
Concrete domains, especially those that allow to compare features with numeric values, 
have long been recognized as a very desirable extension of description logics (DLs), 
and significant efforts have been invested into adding them to usual 
DLs while  keeping the complexity of reasoning in check.
For expressive DLs and in the 
presence of general TBoxes, for standard reasoning tasks like consistency, the most 
general decidability results are for the so-called $\omega$-admissible domains, which are 
required to be dense. Supporting non-dense domains for features that range over integers
 or natural numbers remained largely open, despite often being singled out as a
 highly desirable  extension.
  The decidability of some extensions of $\mathcal{ALC}$
 with non-dense domains has been shown, but existing results rely on powerful machinery that does not allow 
to infer any elementary bounds on the complexity of the problem. In this paper, 
we study an extension of $\mathcal{ALC}$ with a rich integer domain that allows for comparisons 
(between features, and between features and constants coded in unary), and prove that 
consistency can be solved using automata-theoretic techniques in single exponential 
time, and thus has no higher worst-case complexity than standard $\mathcal{ALC}$. Our upper bounds 
apply to some extensions of DLs with concrete domains known from the literature, 
support general TBoxes, and allow for comparing values along paths 
of ordinary (not necessarily functional) roles.

\end{abstract}

\section{Introduction}
\label{sec_introduction}

Concrete domains, especially those allowing to compare features with numeric
values, are a very natural and useful extension of description logics.
Their relevance 
was recognized since the early days of DLs
\cite{DBLP:conf/ijcai/BaaderH91}, and they arise in all kinds of application domains. 
Identifying extensions of DLs that 
keep the complexity of
reasoning in  check has been an ever present challenge for the DL community,
and major research efforts have been devoted to that goal, see
\cite{Lutz02survey} and its references.
The best-known results so far are for the
so-called \emph{$\omega$-admissible domains} which, among other requirements, must be
dense. Decidability and tight complexity results have been established for
several expressive DLs extended with $\omega$-admissible domains based on the
real or the rational numbers. 
However, non-dense numeric domains with the integer or natural 
numbers are not   $\omega$-admissible, and supporting them 
has been often singled out as an open challenge
with significant practical implications
\cite{Lutz02KR,Lutz02survey}.

To our knowledge, there are two decidability results for extensions of
$\mathcal{ALC}$ with non-dense domains based on the integer numbers $\bZ$.
For some domains that support comparisons over the integers,
decidability can be inferred from results on %\mo{
fragments of %} 
CTL$^*$ with  constraints \cite{BozzelliG06}. More recently,
Carapelle and Turhan \cite{CarapelleTurhan16} proved decidability for concrete domains
that have the so-called \emph{EHD-property} (for \emph{existence of a homomorphism
  is definable}), which applies in particular to
$\bZ$ with comparison relations like  `$=$' and  `$<$'.
However, neither of these works allow to infer any elementary
bounds on the complexity of reasoning.
The former result applies the theory of well-quasi-orders to some
dedicated \emph{graphose inequality systems}. The latter result reduces the
satisfaction of the numeric constraints to satisfiability of a formula in a
powerful extension of monadic second order logic with a \emph{bounding
  quantifier}, which has been proved decidable over trees \cite{BojanczykT12}.
In both cases, the machinery stems from  formalisms  stronger than \ALC,
and yields little insight on what is the additional cost of 
the concrete domain. 

In this paper we propose an  automata-theoretic algorithm tightly tailored for
the DL \alcupz, an extension of $\ALCF$ with a domain based on $\bZ$ that follows
the work of Carapelle and Turhan \cite{CarapelleTurhan16}. Not only do we obtain the first elementary complexity upper
bounds, but in fact  we obtain the best
results that we could have hoped for: satisfiability  is
decidable in single exponential time, and thus not harder than for plain
$\ALC$.
The upper bound also applies to other approaches to concrete domains, and 
it extends  
to some domains over the real numbers 
that
include unary predicates 
for asserting 
that
some numbers must be integer or natural. 
Crucially, our setting accommodates general TBoxes, and allows to access the
concrete domain along arbitrary paths of ordinary roles, and not only of
functional ones. 
To our knowledge, this is the first decidability result with both of these
features,
even for $\omega$-admissible domains. 

Our upper bound is obtained using automata-theoretic techniques. Concretely, we
rely on a suitable notion of the tree model property, and build a
non-deterministic automaton on infinite trees that accepts representations of
models of the input.
The key challenge in the presence of TBoxes
comes from verifying whether  
an assignment of integer values along infinite paths exists.
While an infinite path of ever increasing or ever decreasing values always
exists, 
unsatisfiability of non-dense domains can arise from requiring 
an infinite number of integers that are larger than some
integer and smaller than another. 
However, identifying that a given input enforces an infinite sequence of
integers between two bounds may require us to identify, for example, if two
infinite paths in the model meet at ever increasing distances. 
It is far from  apparent how to detect this kind of very non-local behavior in standard
automata, and we could not identify an automata-verifiable condition that precisely characterizes
it. Instead,
we use a condition similar to the one proposed for constraint LTL by
Demri and D'Souza in \cite{DemriD07}, 
which is necessary on all trees, and sufficient on \emph{regular} ones, and
appeal to Rabin's theorem to obtain a sound and complete satisfiability test.
Some proofs are omitted from the body of the paper, and can be found in the appendix.

\paragraph{Related work}
The first DLs with concrete domains were introduced by Baader and Hanschke \cite{DBLP:conf/ijcai/BaaderH91},  
where concrete values are connected
via paths of functional roles, often called  \emph{feature paths}.
They showed that pure concept satisfiability is decidable for concrete
domains $\cD$ that are \emph{admissible}, that is,  satisfiability of conjunctions of
predicates from $\cD$ is decidable, and its predicates are closed under
negation.
Generalizations of this result and tight complexity bounds for specific
settings were obtained in the following years. For example,
 concept satisfiability is $\PSPACE$-complete under certain assumptions
\cite{DBLP:journals/igpl/Lutz02}. 
Adding acyclic TBoxes increases the complexity to $\NEXPTIME$, and general TBoxes
easily result in undecidability \cite{lutzNExpTime}. It remains decidable 
if the paths to concrete domains are restricted to single functional roles
\cite{DBLP:conf/cade/HaarslevMW01}.
Lutz also studied specific concrete domains, for example for
 temporal reasoning \cite{DBLP:journals/ai/Lutz04}, and 
summarized key results in a survey paper \cite{Lutz02survey}. 

Later research relaxed the requirements on the concrete domain, and the most
general results  so far are for extensions of $\ALC(\cC)$ with  \emph{$\omega$-admissible domains}, where concept satisfiability
\wrt to general TBoxes remains decidable   \cite{DBLP:journals/jar/LutzM07}.
However, this and related results assume two key restrictions that we relax in
our work: 
the concrete domain is dense, 
and   
only functional roles occur in the paths connecting to the concrete domains.
Both restrictions are also present in $\bQ$-$\mathcal{SHIQ}$, an
extension of $\mathcal{SHIQ}$ with comparison predicates over the rational
numbers,  for which concept satisfiability \wrt general TBoxes is \EXPTIME-complete.
The logic we consider is closely related to $\bQ$-$\mathcal{SHIQ}$.
It includes the $\ALCF$ fragment of  $\bQ$-$\mathcal{SHIQ}$, but additionally
allows us to replace the rational numbers by integers or naturals. Our
$\EXPTIME$ upper bound   also applies to the extension of $\bQ$-$\mathcal{SHIQ}$
with an $\sfint$ or $\mathsf{nat}$ predicate to make only \emph{some}
 values integer or natural,
and, under certain restrictions, to its extension with arbitrary role paths.

Concerning the latter extension, already the seminal work of Baader and Hanschke \cite{DBLP:conf/ijcai/BaaderH91} points out the
potential usefulness of allowing referral to the concrete domains also along
paths of regular roles,
but this easily results in undecidability.
For example, such an extension of \ALC known as $\ALCP(\cD)$ is
undecidable for any so-called \emph{arithmetic domain} $\cD$  
\cite{lutzNExpTime}. 
However, $\cZ_c$ and its analogue over the real numbers $\cR_c$ are not arithmetic,
and the corresponding DLs $\ALCP(\cZ_c)$ and  $\ALCP(\cR_c)$
do not seem to have been studied before.
By encoding these logics into $\ALCUP(\cZ_c)$ and  
$\ALCUP(\cR_c)$, we prove that their satisfiability problem is
decidable and obtain upper complexity bounds (which are tight under some
restrictions).

Finally, we remark that the extensions of DLs with concrete domains that we
consider here are closely related to constraint temporal logics.
Our logic \alcupz subsumes constraint LTL as defined in \cite{DemriD07},
whose satisfiability problem is \PSPACE\ complete. It is in turn subsumed by
constraint CTL$^*$, and more specifically, by a fragment
of it called  CEF$^+$ in \cite{BozzelliG06}, which unlike full CTL$^*$, has a
decidable satisfiability problem, but for which no tight complexity bounds are
known. Although much of the work on concrete domains in the last decade has
focused on lightweight DLs like DL-Lite (e.g.\ \cite{BaaderBL17,PoggiLCGLR08,SavkovicC12,ArtaleRK12}), 
some advances in the area of constraint CTL 
\cite{DBLP:journals/jcss/CarapelleKL16} have inspired the study of expressive
extensions that had long remained an open problem, like
the ones considered here \cite{CarapelleTurhan16}. 

\section{The \texorpdfstring{$\ALCUP(\cZ_c)$}{ALCF-P(Zc)} description logic}
\label{sec_ALCPZ_prelim}

The DL $\ALCUP(\cD)$ was introduced by Carapelle and Turhan \cite{CarapelleTurhan16}
for arbitrary domains $\cD$. Here we instantiate this DL with the concrete domain $\cZ_c$ that is defined as $\bZ$ equipped with
the standard binary equality and comparison relations `$=$' and `$<$', 
as well as a family of unary relations for comparing with an integer
constant. 

\begin{definition}[Syntax of $\ALCUP(\cZ_c)$]
  Let $\mathsf{Reg}$ be a countably infinite set of \emph{registers}
  (also known as \emph{concrete features}). A \emph{register term} is
  an expression of the form $S^{k} x$, where $x\in \mathsf{Reg}$ and
  $k\geq 0$ is an integer. An \emph{atomic constraint} is an
  expression of the form (i) $t=t'$, (ii) $t<t'$, or (iii) $t= c$,
  where $t,t'$ are register terms, and $c\in\bZ$. A \emph{(complex)
    constraint} $\Theta$ is an expression built from atomic
  constraints using the Boolean connectives $\lnot, \land$ and
  $\lor$. The depth of $\Theta$  (in symbols, $depth(\Theta)$) is the maximal $d$ such that some
  register term $S^{d}x$ appears in $\Theta$.

    Let $\sfNC$ and $\sfNR$ be  countably infinite sets of
  \emph{concept} and \emph{role names}, respectively. We
  further assume an infinite set $\sfNF \subseteq \sfNR$ of
  \emph{functional} role names.  A \emph{role path $P$} is any finite
  sequence $r_1 \cdots r_n$ of role names, with $n\geq 0$. We use
  $|P|$ to denote the length of $P$, i.e.\ $|P|=n$. Note that the
  empty sequence is also a role path, which we denote with
  $\epsilon $.

  $\ALCUP(\cZ_c)$ concepts are defined as follows:
  \begin{align*}
    C := A \mid \neg C \mid (C \sqcap C) \mid \exists r. C \mid \exists P. \lpath \Theta\rpath
  \end{align*}
  where $A \in \sfNC$, $r \in \sfNR$, $P$ is a role path, and $\Theta$
  is a constraint with $depth(\Theta) \leq |P|$.
    We use $C\sqcup D$ as
  an abbreviation of $\neg ( \neg C \sqcap \neg D) $, and
  $\forall r. C$ as an abbreviation of $\neg \exists r.\neg
  C$. Moreover, we use $\forall P. \lpath \Theta\rpath$ instead of
  $\neg \exists P. \lpath \neg \Theta\rpath $.
  Concepts of the form $D=\exists P. \lpath \Theta \rpath$ and
  $D=\forall P. \lpath \Theta \rpath$ are  called \emph{path
    constraints}, and we let $depth(D)=|P|$.

  A \emph{TBox} $\cT$ is any finite set of
  \emph{axioms}, where each axiom has the form
  $C \sqsubseteq D$ for some concepts $C$ and $D$.

  (Plain) $\ALCF$ concepts and TBoxes are defined as in   $\ALCUP(\cZ_c)$ but do not
  allow path constraints. 
\end{definition}
 
We can now define the semantics of the considered DL.
\begin{definition}[Semantics]
  \label{def_Z_interpretation}
  An \emph{interpretation} is a tuple
  $\cI = (\Delta^\cI, \cdot^\cI, \beta)$, consisting of a non-empty
  set $\Delta^\cI$ (called \emph{domain}), a \emph{register
    function} 
  $\beta: \Delta^\cI \times \sfReg \rightarrow \mathbb{Z}$, and a
  \emph{(plain) interpretation function} $\cdot^\cI$ that assigns
  $C^{\cI} \subseteq \Delta^{\cI}$ to every concept name 
  $C \in \sfNC$, 
  $r^{\cI} \subseteq \Delta^{\cI} \times \Delta^{\cI}$ to every role
  name $r \in \sfNR$. We further require that
  $\{(v,v'),(v,v'')\}\subseteq r^{\cI}$ implies $v'=v''$ for all
  $r \in \sfNF$.
	Role paths denote  tuples of elements. %} 
  For a role path $r_1\cdots r_n$, we define
  $(r_1  \cdots r_n)^\cI$ as the set of all tuples
  $(v_0,\ldots,v_n) \in \Delta^{n+1}$ such that
  $(v_{0},v_1) \in r_1^\cI,\ldots, (v_{n-1},v_n) \in r_n^\cI$.

  For an interpretation $\cI$ and a tuple $\vec{v}=(v_0,\ldots,v_n)$ of
  elements in $\Delta^{\cI}$, 
  we define the following, where  $\theta \in \{=,<\}$ and $c \in \bZ$:
  \begin{itemize}
  \item $\cI,\vec{v}\models \theta(S^{i} x,  S^{j} y)$ iff $\beta(v_i,x) \theta \beta(v_j,y)$;
  \item $\cI,\vec{v}\models S^{i} x = c$ iff $\beta(v_i,x)= c$;
  \item $\cI,\vec{v}\models \Theta_1\land \Theta_2 $ iff $\cI,\vec{v}\models \Theta_1$ and 
    $\cI,\vec{v}\models \Theta_2$;
    \item $\cI,\vec{v}\models \Theta_1\lor \Theta_2 $ iff $\cI,\vec{v}\models \Theta_1$ or
    $\cI,\vec{v}\models \Theta_2$; 
    \item $\cI,\vec{v}\models \neg \Theta $ iff $\cI,\vec{v}\not \models \Theta$.
  \end{itemize}

Now %} 
the function $\cdot^{\cI}$ is  extended to complex concepts as follows:
  \begin{itemize}
  \item $(\neg C)^{\cI} = \Delta^{\cI} \setminus C^{\cI}$ and  $(C \sqcap D)^{\cI} = C^{\cI} \cap D^{\cI}$,
    \item  $(\exists r. C) ^{\cI} = \{v \mid (v,v')\in r^{\cI}, v' \in
  C^{\cI}\}$, and
\item  $(\exists P. \lpath \Theta \rpath )^\cI=\{u\mid (u,\vec{v})\in
  P^\cI\mbox{ and }\cI,(u,\vec{v})\models\Theta \}$.
\end{itemize}

An interpretation $\cI$ is a \emph{model} of a TBox $\cT$, if
$C^{\cI}\subseteq D^{\cI}$ for all $C\sqsubseteq D\in \cT$.
We say that a concept $C$ is \emph{satisfiable \wrt $\cT$} if there is a model
$\cI$  of $\cT$ with $C^{\cI} \neq \emptyset$. 
\end{definition}

\begin{example}
  The TBox with the axiom
  $\top \sqsubseteq \exists r. \lpath S^0 x < S^1 x \rpath$ enforces
  an infinite chain of objects whose $x$ registers store increasing
  integer values. This witnesses that $\ALCUP(\cZ_c)$ does not enjoy
  the finite model property. 
\end{example}

\paragraph{Tree model property}

The automata-based techniques we employ in this paper rely on the
\emph{tree model property} of $\ALCUP(\cZ_c)$. We recall the
definition of tree-shaped models (cf.~\cite{CarapelleTurhan16}). For $n \geq 1$, let 
$[n] = \{1, \ldots, n\}$.  We say
$\cI = (\Delta^{\cI}, \cdot^{\cI}, \beta)$ is \emph{tree-shaped} if
$\Delta^{\cI} = [n]^{\star}$ for some $n\geq 1$, and for every
$u, v \in \Delta^\cI$, we have that $(u,v) \in r^{\cI}$ for some
$r \in \sfNR$ iff $v = u \gamma$ for some $\gamma \in [n]$.
Let $\gamma_1, \ldots, \gamma_k \,{\in}\, [n]$. If
$u\gamma_k \,{\in}\, \Delta^{\cI}$, we call $u$ the \emph{parent} of
$u\gamma_k$, and if
$v=u \gamma_1 \cdots \gamma_k \,{\in}\, \Delta^{\cI}$, we call $u$ the
\emph{$k$-th ancestor} of $v$. Such $\cI$ is called an \emph{$n$-tree
  (interpretation)}.

The following theorem will allow us to focus on $n$-trees for our technical developments:
\begin{theorem}[Carapelle and Turhan, \cite{CarapelleTurhan16}]
\label{th_n_tree_model}
Let $\hat{\cC}$, $\hat{\cT}$ be in negation normal form, where $d$ is the maximal depth of an existential path constraint in $\hat{\cC}$ or $\hat{\cT}$,
and $e$ the number of existentially quantified subconcepts in $\hat{\cC}$ or $\hat{\cT}$. 
If $\hat{\cC}$ is satisfiable \wrt $\hat{\cT}$, then it has an $n$-tree model where \mbox{$n = d \cdot e$.}
\end{theorem}

\section{A tight upper bound for satisfiability} 
\label{sec_upper_bound}
In this section we present our main result: an algorithm for deciding \alcupz concept
satisfiability \wrt to general TBoxes in \emph{single exponential time}.
The algorithm uses automata on infinite trees, and reduces  the satisfiability
test to the emptiness of a suitable automata.
But first we bring \alcupz concepts  and TBoxes into a simpler shape that facilitates the later developments.

\subsection{Atomic normal form}
\label{subsec_normalization}

Here we go from a concept $\cC'$ and TBox $\cT'$ in general
form to equisatisfiable $\hat{\cC}$ and $\hat{\cT}$ in \emph{atomic normal
form}, where the path constraint are of length $1$ and the register constraints are atomic.
This conversion relies on the tree model property of $\ALCUP$ and on $\cZ_c$ being negation-closed. 
That is, the negation of an atomic relation can be expressed without negation via other relations;
for $<$ and $=$ negation can be removed using only one disjunction, and for $=c$ negation can be removed using
one conjunction with one disjunction and one fresh register name.

\begin{definition}[Atomic normal form]
An $\ALCUP(\cZ_c)$-concept is in \emph{atomic normal form (ANF)} if for
every $\exists P. \lpath \Theta \rpath$ and $\forall P. \lpath \Theta \rpath$ that appears in it,  
$\Theta$ is an atomic constraint and $|P|\leq 1$. 
A TBox $\cT$ is in ANF if the TBox-concept $\bigsqcap_{C \sqsubseteq D \in \cT}(\neg C \sqcup D)$ is in ANF.
\end{definition}

\begin{lemma}
\label{lem_new_normal_form}
Let $\cC'$ and $\cT'$ be a concept and a TBox in $\ALCUP(\cZ_c)$. Then
$\cC'$ and $\cT'$ can be transformed in polynomial time into $\cC$ and $\cT$  in ANF 
such that  $\cC$ is satisfiable w.r.t.\ $\cT$ iff $\cC'$
is satisfiable w.r.t.\ $\cT'$.
\end{lemma} 
\begin{proof}[Proof sketch.]
% Step 1 
We can convert $\cC'$ and $\cT'$ to negation normal form and then remove
 negation from atomic constraints using $\wedge$, $\vee$, and at most
one fresh register name per constraint, all in linear time. Therefore we assume that $\cC'$ and $\cT'$ are negation free.
Next, relying on the tree model property, we copy at each node $u$ the
registers of its ancestors that may occur in the same constraints as $u$'s own 
registers.
For this we propagate the register values of the ancestors one step at a time 
with axioms 
\begin{align*}
\top \sqsubseteq \forall r.
  \lpath S^1 x_{i,P}^{k} = S^0 x_{i,P}^{k-1} \rpath
\end{align*}
We define a  TBox $\cT_{\mathsf{prop}}$
that contains such an axiom for each appropriate role name $r$ and (fresh) register names associated with role paths $P$
and depth $k$ of path constraints used in
$\cC'$ and $\cT'$. 
Note that along every path $P$, the TBox $\cT_{\mathsf{prop}}$ propagates values 
into copy-registers associated with \emph{all} paths appearing in $\cC'$ or $\cT'$, not just into the copy-registers associated with $P$. 
We will later restrict our attention to the relevant registers depending on context.
The following claim is proved with a straightforward inductive construction:
\begin{claim}
\label{cl_part_I}
Every tree model of $\cC'$ w.r.t.\
$\cT' \cup \cT_{\mathsf{prop}}$ contains a tree model of $\cC'$
w.r.t.\ $\cT'$, and every tree model of $\cC'$ w.r.t.\ $\cT'$
can be expanded to a tree model of $\cC'$ w.r.t.\
$\cT' \cup \cT_{\mathsf{prop}}$.
\end{claim}	

% Step 2
For a role path $P$ and an atomic constraint $\Theta$, let $\mathsf{loc}(\Theta,P)$ denote the constraint obtained from $\Theta$
by replacing each occurrence of $S^j x_{i}^{0}$ with $ S^0 x_{i,P}^{|P|-j}$. 
In the next step, we create some ``test'' concept names and axioms
that will allow to check whether a given constraint is satisfied in a
certain path in a tree model. 
For each (sub)constraint $\Theta$ and a role path $P$ that appear in $\cC'$ or $\cT'$,
take a fresh concept name
$T_{P,\Theta}$ and add to a TBox $\cT_{\mathsf{loc}}$ the following axioms (recall that $\cC'$ and $\cT'$ are negation free):
\begin{enumerate}[label={({\bf A}\textsubscript{\arabic*}}),ref=({\bf A}\textsubscript{\arabic*}),series=conditions]
\item 
%\label{it_and}
$T_{P,\Theta} \equiv T_{P,\Theta_1} \sqcap T_{P,\Theta_2}$ if $\Theta = \Theta_1\land  \Theta_2$ 
\item 
%\label{it_or}
$T_{P,\Theta} \equiv T_{P,\Theta_1} \sqcup T_{P,\Theta_1}$  if $\Theta = \Theta_1\lor \Theta_2$ 
\item 
%\label{it_loc}
$T_{P,\Theta} \equiv \exists \epsilon \lpath \mathsf{loc}(\Theta,P) \rpath$ if  $\Theta$ is an atomic constraint.
\end{enumerate}
We make two claims about combining $\cT_{\mathsf{prop}}$ with $\cT_{\mathsf{loc}}$.
The first is that we can continue expanding the initial tree model:
\begin{claim}
\label{cl_part_II}
Every tree model of $\cC'$ w.r.t.\ $\cT' \cup \cT_{\mathsf{prop}}$ can be expanded to a tree model of $\cC'$ w.r.t.\
$\cT' \cup \cT_{\mathsf{prop}} \cup \cT_{\mathsf{loc}}$, and every tree model of 
$\cC'$ w.r.t.\
$\cT' \cup \cT_{\mathsf{prop}} \cup \cT_{\mathsf{loc}}$
is a tree model of $\cC'$ w.r.t\ $\cT' \cup \cT_{\mathsf{prop}}$.
\end{claim}
The above claim follows by induction on $\Theta$.
Next, we claim that $\cT_{\mathsf{prop}} \cup \cT_{\mathsf{loc}}$ indeed
relates the satisfaction of path constraints to membership in the test concepts:
\begin{claim}
\label{cl_part_II_a}
Let $\cJ$ be a tree model of $\cT_{\mathsf{prop}} \cup \cT_{\mathsf{loc}}$, 
and let $P$ be a role path and $\Theta$ a constraint appearing in $\cC'$ or $\cT'$.
Then it holds that
\begin{enumerate}
\item $\cJ$ contains 
	a $P$-path
  $e_0,\ldots, e_{|P|}$ and if $e_{|P|}\in T_{P,\Theta}^{\cJ}$, then the path
  $e_0,\ldots, e_{|P|}$ satisfies $\Theta$ in $\cJ$;
\item if $\cJ$ contains a $P$-path  $e_0,\ldots, e_{|P|}$ that satisfies $\Theta$, then $e_{|P|} \in T_{P,\Theta}^{\cJ}$.
\end{enumerate}
\end{claim}

Now we are ready to rewrite $\cC'$ and $\cT'$ into ANF using the locally available copy-registers
and the test concept names;
Given a concept $D$ and a role path
$P= r_1\cdots r_n$, we write $\exists P.D $ as shorthand for $\exists r_1 (\exists r_2(\cdots (\exists r_n.D)\cdots))$,
and similarly for $\forall P. D$.
Let $\cC$ and $\cT^{*}$ be
obtained from $\cC'$ and $\cT'$, respectively, by replacing every concept
$\exists P.\lpath\Theta \rpath $ by $\exists P.T_{P,\Theta} $ and
every $\forall P.\lpath\Theta \rpath $ by $\forall P.T_{P,\Theta}$. 
Our desired normalization is $\cC$ equipped with the TBox
$\cT=\cT^{*}\cup \cT_{\mathsf{prop}} \cup \cT_{\mathsf{loc}}$.

Given a tree model of $\cC'$ w.r.t.\ $\cT'$, by chaining Claim~\ref{cl_part_I} and Claim~\ref{cl_part_II}, we get a tree model $\cJ$ of $\cC'$ w.r.t.\
$\cT' \cup \cT_{\mathsf{prop}} \cup \cT_{\mathsf{loc}}$, and by applying Claim~\ref{cl_part_II_a}
we get that 
\[
(\exists P.\lpath\Theta \rpath)^\cJ  = (\exists P.T_{P,\Theta})^\cJ, \quad
(\forall P.\lpath\Theta \rpath)^\cJ  = (\forall P.T_{P,\Theta})^\cJ
\]
Hence $\cJ$ is also a tree model of $\cC$ w.r.t.\ $\cT$.

Given a tree model $\cJ$ 
of $\cC$ w.r.t.\ $\cT$, again by applying Claim~\ref{cl_part_II_a} we get that $\cJ$
is also a tree model of $\cC'$ w.r.t.\ $\cT' \cup \cT_{\mathsf{prop}} \cup \cT_{\mathsf{loc}}$ (and
in particular w.r.t.\ $\cT'$).

\end{proof}

\subsection{Abstractions and constraint graphs}
\label{sec_tree_rep}

To check satisfiability of $\cC$ \wrt $\cT$, we follow the approach
of \cite{CarapelleTurhan16} and split the task into two checks:
a satisfiability check for an \emph{abstracted} version of $\cT$, $\cC$, which
is in plain $\ALCF$,
and an \emph{embeddability check} for so-called \emph{constraint graphs}.
We recall the definitions of abstracted $\ALCUP(\cZ_c)$ concepts and
constraint graphs from \cite{CarapelleTurhan16}, adapted to our context.

\begin{definition}[Abstraction]
\label{def_abstraction}
Consider a path constraint $E = \exists r. \lpath \Theta \rpath$, 
where $\Theta$ is an atomic constraint.
Let $B \in \mathbf{B}$ be a fresh concept name, which we call the \emph{placeholder} of $\Theta$. 
The \emph{abstraction} of $E$ is defined as
	\[
	E_a = \exists r. B
	\]
The abstraction of a universal path constraint $E' = \forall r. \lpath \Theta' \rpath$ is analogous.
If $r = \epsilon$, then the abstraction is simply $B$.

The abstractions of concepts and TBoxes given in ANF are the (plain $\ALCF$)
concepts and TBoxes obtained by replacing all path constraints with their
abstracted versions.
\end{definition}

Let $\cC_a$, $\cT_a$ be the abstractions of $\cC$ and $\cT$, respectively. 
Let $\sfReg_{\cC,\cT}$ be the set of register names used in $\cC$ and $\cT$.

The constraint graph of a plain tree-shaped interpretation 
indicates how the values of its registers participate in relevant
relations.
Comparisons and equalities between registers are represented as graph edges, and equalities with constants are stored as node labels.

\begin{definition}[Constraint graph]
\label{def_constraint_graph}
Let $\cI_a = (\Delta^{\cI_a}, \cdot^{\cI_a})$ be a plain tree-shaped interpretation of $\cC_a,\cT_a$.
The \emph{constraint graph} of $\cI_a$ is the directed partially labeled graph
$\cG_{\cI_a} = (V,E,\lambda)$ where $V = \Delta^{\cI_a} \times \mathsf{Reg}_{\cC,\cT}$, $\lambda: V \rightarrow 2^{\mathbf{B}}$, and  
$E = E_{<} \cup E_{=}$ is such that, for every $(v,y),(u,x) \in V$, 
	\begin{enumerate}
		\item 
		$((v,y),(u,x)) \in E_{<}$ if and only if either
		\begin{compactitem}
			\item
			$u = v$, $v \in B^{\cI_a}$ and $B$ is a placeholder for
                          $S^0 y < S^0 x$, 
		        \item 
			$u$ is the parent of $v \in B^{\cI_a}$ and $B$ is a placeholder for	$S^1 y < S^0 x$, or
			\item
			$v$ is the parent of $u \in B^{\cI_a}$ and $B$ is a placeholder for	$S^0 y < S^1 x$.
		\end{compactitem}
		\item
		$((v,y),(u,x)) \in E_{=}$ if and only if  
		\begin{compactitem}
			\item 
			$u = v$, $v \in B^{\cI_a}$ and $B$ is a placeholder for	$S^0 y = S^0 x$,
			\item
			$u$ is the parent of $v \in B^{\cI_a}$ and $B$ is a placeholder for	$S^1 y = S^0 x$, or
			\item
			$v$ is the parent of $u \in B^{\cI_a}$ and $B$ is a
                          placeholder for	$S^0 y = S^1 x$.  
		\end{compactitem}
	\end{enumerate} 
In addition, for a placeholder $B$ for $S^0 x = c$, we have that
$B \in \lambda(u,x)$ if and only if $u \in B^{\cI_a}$. 
\end{definition}
When the interpretation is clear from context, we write $\cG$.

We say a constraint graph $\cG$ is \emph{embeddable} into $\cZ_c$ if there is an integer assignment $\kappa$ to the vertices $V$ of 
$\cG$ such that for every $(u,x),(v,y) \in V$, if $((v,y),(u,x)) \in E_<$ then
$\kappa(u,w) < \kappa(v,y)$ (and similarly for $E_=$), and if $B \in \lambda(u,x)$
%% and $B$
is a placeholder for $S^0 x = c$, then $\kappa(u,x) = c$. 

\begin{figure}
  \centering
  \includegraphics[scale=.57]{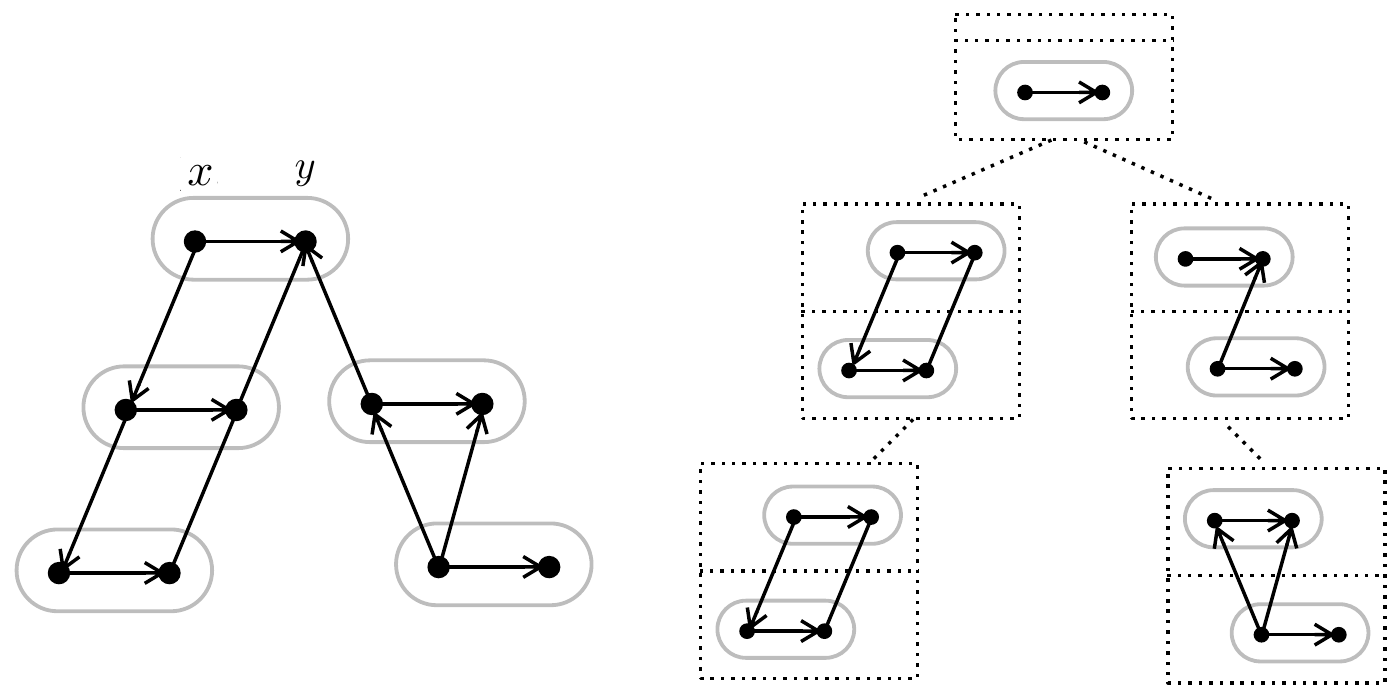}
  \caption{A constraint graph with registers $x$ and $y$ for each logical element (left) and its tree representation (right). 
The label $\lambda$  is empty for all nodes and not shown. 
} 
  \label{fig:constrgraph}
\end{figure}

\begin{example}
\label{ex_constraint}
  The left hand side of Figure~\ref{fig:constrgraph} shows a constraint graph
  for an interpretation where each element satisfies $\exists \epsilon. \lpath S^0 x < S^0 y \rpath$, 
	 the root and its right child satisfy $\exists r. \lpath S^0 y  > S^1 x \rpath$, the root and its left child 
	satisfy $\exists r. \lpath (S^0 x < S^1 x) \wedge (S^0 y = S^1 y) \rpath$, and the right child of the root additionally
	satisfies $\exists r. \lpath (S^0 x\,{<}\,S^1 x) \wedge (S^0 y < S^1 x)\rpath$. 

  This constraint graph is embeddable into $\bZ$. Consider, however, an
  infinite interpretation in which the leftmost branch of the constraint graph
  repeats infinitely. Then we would have
	paths from the $x$ to the $y$ register of the root
	involving any finite number of edges from $E_<$, which would imply 
	that the integer values assigned to these registers must have infinitely many different integer values between them.
	Thus in this case, the graph would not be embeddable.
  \end{example}

Abstractions and constraint graphs allow us to reduce 
$\ALCUP(\cZ_c)$  satisfiability to two separate checks:

\begin{theorem}[Carapelle and Turhan, \cite{CarapelleTurhan16}]
\label{th_split_checks}
$\cC$ is satisfiable \wrt $\cT$ if and only if there is a tree-shaped $\cI_a \models_{\cT_a} \cC_a$ such that
$\cG_{\cI_a}$ is embeddable into $\cZ_c$.
\end{theorem}

\subsection{Embeddability condition} 
\label{subsec_embd_cond}
Our first aim is to test  embeddability using tree automata.
For this, we represent 
(augmented) constraint graphs as trees.
In a nutshell, our \emph{tree representations} are tree decompositions
(cf.\,\cite{Diestel}) 
where each bag holds
the subgraph induced by a logical element $u \in \Delta^{\cI_a}$ and its parent.
Since our constraints have maximal depth $1$,
we can do this using an alphabet that stores information about two  logical
elements: a parent at the top ($\mathrm{top}$) and a child at the bottom
($\mathrm{bot}$).

This is illustrated in Figure \ref{fig:constrgraph}, where 
the tree representation of the constraint 
graph is shown on the right hand side.
Note that the bottom part of the label of each vertex induces the same graph
at the top part of the label of each child, and that the label of the root vertex 
has no top row.

For these representations we use two copies $x^{\mathrm{top}}$ and
$x^{\mathrm{bot}}$ of each $x \in \sfReg_{\cC,\cT}$, and call the respective
sets $\sfReg_{\cC,\cT}^{\mathrm{top}}$ and $\sfReg_{\cC,\cT}^{\mathrm{bot}}$.
The relevant information about (in)equalities with constants is
stored as a partial labeling in these tree representations.
Let $c_{0}$ be the smallest integer used in either $\cC$ or $\cT$ and let $c_\alpha$ be the largest. If no
integers were used, set $c_0 = c_\alpha = 0$.
Denote by $[c_0,c_{\alpha}]$ the range of integers between $c_0$ and $c_{\alpha}$, inclusive.
Let $\mathbf{U} = \{U_{<c_0}, U_{c_{\alpha} < }\} \cup \{U_{c} \mid  c \in [c_0, c_{\alpha}]\}$ be fresh labels.
Let $\Sigma$ be the set of partially $\mathbf{U}$-labeled graphs where the vertex set is either exactly  
$V = \mathsf{Reg}_{\cC,\cT}^{\rtop} \cup \mathsf{Reg}_{\cC,\cT}^{\mathrm{bot}}$ 
or $V = \mathsf{Reg}_{\cC,\cT}^{\mathrm{bot}}$, 
and $E = E_{<} \cup E_{=}$.

\begin{definition}[Tree representation of constraint graph]
\label{def_tree_representation}
Let $\mathcal{G}$ be the constraint graph of some plain interpretation $\cI_a$. 
For $u \in \Delta^\mathcal{I}$ with  parent $v$, define $X(u)$ as the subgraph of $\mathcal{G}$ induced by 
$\{(u,x) \mid x \in \mathsf{Reg}_{\mathcal{C},\mathcal{T}}\} \cup \{(v,x) \mid x \in \mathsf{Reg}_{\mathcal{C},\mathcal{T}}\}$.
For $u = \varepsilon$, define $X(u)$ as the subgraph of $\mathcal{G}$ induced by $\{(u,x) \mid x \in \mathsf{Reg}_{\mathcal{C},\mathcal{T}}\}$.
Let $Y(u)$ be the following partially $\mathbf{U}$-labeled graph:
\begin{itemize}
	\item 
	The vertices of $Y(u)$ are obtained from $X(u)$ by renaming $(u,x)
        \mapsto x^{\mathrm{bot}}$ and $(v,x) \mapsto x^{\rtop}$ for every $x
        \in \sfReg_{\cC,\cT}$. 
	\item
	The edges of $Y(u)$ are exactly those of $X(u)$ (under the renaming).
	\item
	We have $U_c(x^{\rbot})$ if and only if $(u,x)$ is labeled with a
        placeholder for $=c$, and similarly for $x^{\rtop}$. 
\end{itemize} 

The \emph{tree representation} $\mathrm{Tr}(\mathcal{G})$ of $\mathcal{G}$ is the tree over $\Sigma$
where $\mathrm{Tr}(\mathcal{G})(u) = Y(u)$.
\end{definition}

Rather than considering tree representations where  nodes are labeled with
arbitrary graphs from $\Sigma$, it will be convenient to consider trees over a
restricted alphabet that contains only graphs that have been enriched with
implicit information in a  maximal consistent way.

\begin{definition}[Frame]
\label{def_frame}
A \emph{frame} is a graph in $\Sigma$ such that: 
\begin{enumerate}
	\item 
	there is an edge between every pair of vertices
	\item
	there are no strict cycles, i.e.\ no cycles that include an edge from $E_{<}$
	\item
	if $e_=(x,y)$ then also $e_=(y,x)$
	\item 
	every vertex must have exactly one of the labels in $\mathbf{U}$
	\item
	\label{cond_equality}
	$e_=(x,y)$ iff $x$ and $y$ have the same label from $\mathbf{U}$.
	\item
	  If $e_<(x,y)$ then either       
	\begin{inparaenum}
		\item 
		$U_{<c_0}(x)$, or
		\item
		$U_{c_{\alpha} <}(y)$, or
		\item
		$U_{c_i}(x)$ and $U_{c_j}(y)$ with $c_i,c_j \in [c_0,
                  c_\alpha]$ and $c_i < c_j$. 
	\end{inparaenum}
\end{enumerate}
We denote the alphabet of frames by $\Sigma_{\fr}$.
\end{definition}

\begin{definition}[Framified constraint graph]
\label{def_framified}
We say that an augmentation $\cG_{\mathrm{fr}}$ of a constraint graph $\cG$ is a \emph{framified constraint graph}
if its tree representation is over $\Sigma_{\mathrm{fr}}$.
\end{definition}
Note that a constraint graph may have multiple framifications; e.g. if not all
registers are compared to a constant. It may also have no framifications; e.g. if it contains a strict cycle.
In fact, a framification may not introduce strict cycles. %%, either.
\begin{lemma}
\label{lem_no_cycles_in_framification}
Let $\cG_\fr$ be a framified constraint graph. Then there are no strict cycles in $\cG_\fr$.
\end{lemma}
\begin{proof}[Proof Sketch]
We show by induction that if a strict cycle spanning the registers of $k$ logical elements
exists, then due to the existence of an edge between every pair of vertices, there is also a strict cycle spanning the registers of $k-1$ logical elements.
Repeating until the strict cycle spans at most 2 logical elements, we obtain a contradiction to Def.\,\ref{def_frame}.
\end{proof}

An embeddable constraint graph can always be framified. 
\begin{observation}
\label{obs_framification}
Let $\cG$ be an embeddable constraint graph. Then there exists a framification of $\cG$.
\end{observation}

In the tree representation of (framified) constraint graphs,
the $\rbot$ part of a vertex 
coincides with the $\rtop$  of its children. 

\begin{definition}[Consistent frames] 
\label{def_consistent_frames}
Let $\sigma_1,\sigma_2 \in \Sigma_\fr$. We call the pair
$(\sigma_1,\sigma_2)$ \emph{consistent}
if the following are equal:\\
$\bullet$ the subgraph induced by the $\sfReg_{\cC,\cT}^\rbot$ vertices of $\sigma_1$,
and\\
$\bullet$  the result of renaming each $x^\rtop$ to $x^\rbot$
in the subgraph induced by  the $\sfReg_{\cC,\cT}^\rtop$ vertices of
$\sigma_2$.
\end{definition}

Not every tree $T$ over $\Sigma_\fr$ corresponds to a framified constraint graph, but when
all parent-child pairs are consistent, we can refer to the framified constraint graph $T$ represents:

\begin{definition}
  Let $T$ be a tree over $\Sigma_\fr$.
We call $T$ \emph{consistent} if
the pair $(T(v),T(vh))$ is consistent
for every $v \in [n]^\star$ and every $h \in [n]$. We
denote by $\cG_\fr^T$ the framified constraint graph  with
$\mathrm{Tr}(\cG_\fr^T) = T$, and say that $T$ represents 
$\cG_\fr^T$. 
\end{definition}

We use the following terminology for talking about paths.

\begin{definition}\label{def_path_along}
Let $\bfw = \gamma_1 \gamma_2 \cdots$ be a finite or infinite word over $[n]$ and let $u \in [n]^\star$.
\emph{A path along $\bfw$ from $(u,x)$} is a path of the form $p = (u,x) - (u\gamma_1, x_1) - (u\gamma_1 \gamma_2, x_2) - \cdots$.

An infinite path $p : \mathbb{N} \rightarrow \Delta \times \mathsf{Reg}$ is a \emph{forward path}
if for every $n \in \mathbb{N}$, there is an edge from $p(i)$ to $p(i+1)$. 
It is a \emph{backward path} if for every $n \in \mathbb{N}$, there is an edge from $p(i+1)$ to $p(i)$.
The \emph{strict length} of a finite path $p$ is the number of strict edges in $p$. 
For an infinite path $p$, we say that \emph{$p$ is strict} if it has
infinitely many strict edges. 
\end{definition}

The following condition on framified constraint graphs will be crucial to
deciding embeddability:

\begin{description}
\item[$(\condsym)$]
There are no $(u,x),(u,y) \in \Delta \times \mathsf{Reg}_{\cC,\cT}$ in $\mathcal{G}_{\mathrm{fr}}$ 
for which we have that: there exists an infinite $\mathbf{w} \in
[n]^\omega$, and 
\begin{enumerate}
	\item 
	an infinite forward path $f$ from $(u,x)$ along $\mathbf{w}$, and 
	\item
	an infinite backward path $b$ from $(u,y)$ along $\mathbf{w}$
\end{enumerate}
such that $f$ or $b$ is strict, and such that for every $i \in \mathbb{N}$, there is a strict edge from $f(i)$ to $b(i)$.
\end{description}

Indeed,  $(\condsym)$ is a necessary  condition  for embeddability: 
\begin{lemma}
\label{lem_condition_necessary}
If a constraint graph is embeddable, then it satisfies the condition $(\condsym)$.
\end{lemma}
\begin{proof}[Proof Sketch]
By contradiction. The existence of such a pair and paths would imply that the integers assigned to $(u,x)$ and $(u,y)$ 
have infinitely many different integers between them, since a path of strict length $k$ from $(u,x)$ to $(u,y)$
implies there being a difference of at least $k$ between their assigned values.
\end{proof}

Unfortunately, it is not sufficient in general.

\begin{example}[From \cite{DemriD07}]
  Figure~\ref{fig:ladder} shows an example of a constraint graph; to avoid clutter, we omitted the edges 
  augmented in its framification.
  It satisfies the condition $(\condsym)$ since there is no path with
infinitely many strict edges.
It is not embeddable into $\bZ$: indeed, for any
$n$, there is a path with at least $n$ strict edges between $x$ and $z$.
\end{example}

\begin{figure}
  \centering
  \includegraphics[scale=.57]{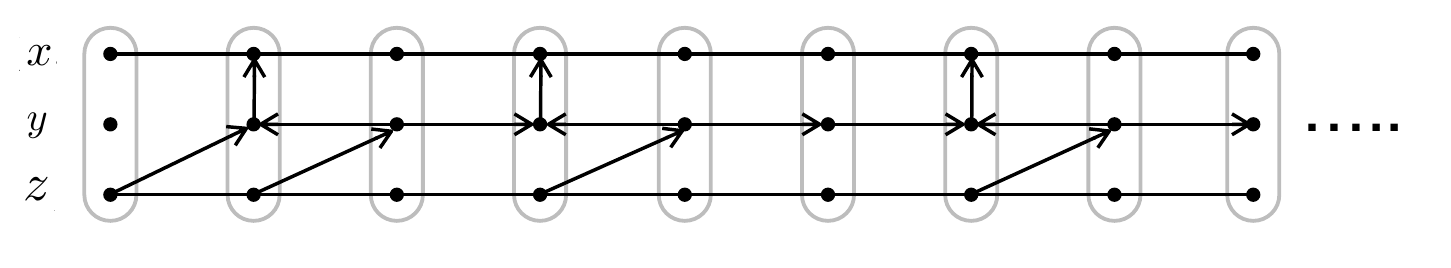}
  \caption{A constraint graph that satisfies $(\condsym)$
but is not embeddable into  $\bZ$.}
  \label{fig:ladder}
\end{figure}

Nonetheless, the condition will allow us to effectively test embeddability,
since it is  sufficient for \emph{regular} framified constraint graphs.

\begin{definition}
\label{def_regular_tree}
\label{def_regular_constraint_graph}
For an $n$-tree $T$ over $\Sigma$, the \emph{subtree rooted at $w \in [n]^\star$} is the tree $\left.T\right|_{w}(v) = T(wv)$
for all $v \in [n]^\star$.  

We say that an $n$-tree $T$ over $\Sigma$ is \emph{regular} if the set 
$\{\left. T \right|_{u} \mid u \in [n]^\star\}$ of subtrees of $T$ is finite.
We say that a constraint graph is regular if its tree representation is regular.
\end{definition}
\noindent
The next key lemma is the most technical result of the paper.
\begin{lemma}
\label{lem_condition_suff_for_regular}
Let $\cG_{\fr}$ be a regular framified constraint graph. If $\cG_{\fr}$ satisfies $(\condsym)$, then it is embeddable.
\end{lemma}
\begin{proof}[Proof sketch]
Let $\cG_{\fr}$ be a regular framified constraint graph which is not embeddable. 
The heart of the proof is showing that there is a pair $(u,x), (u,y)$ and a finite path 
from $(u,x)$ to $(u,y)$ of a certain shape and positive strict length,
which may be extended indefinitely to obtain the desired $f$ and $b$.

First, we show that there is a pair $(u,x), (u,y)$ such that for any $m \in \bN$, there is a path from $(u,x)$ to $(u,y)$ 
of strict length at least $m$ which only involves vertices whose logical
element has the prefix $u$, that is, the path only involves vertices in the subtree 
rooted at $u$.
However, the path may move down and up this subtree arbitrarily.
We then use framification to describe a path $p'$ of a specific shape, which first goes down along some $w$ and then goes back up.
The path $p'$ may have reduced strict length, but we show a lower bound on the strict length of $p'$ which is a function of $m$.

Next we use regularity to argue that for large enough $m$, the path $p'$ becomes long enough that it essentially starts repeating itself, 
thus allowing us to extend it indefinitely (as well as the word $w$ it runs long)
to obtain the desired forward and backward paths; $f$ is obtained by
concatenating the downward portion of $p'$ and $b$ is obtained by
concatenating the upward portion. The strict edges between $f$ and $b$ are
given by the framification.
\end{proof}

\subsection{A Rabin tree automaton for embeddability}
\label{subsec_automata_const}
We still face two hurdles:
verifying ($\condsym$), and ensuring that satisfiable $\cC$ \wrt $\cT$ have a model 
with a \emph{regular} constraint graph. 
We overcome both by using Rabin's tree
automata. 

Recall that the trees are over the alphabet of frames $\Sigma_\fr$ and are of degree $n$ i.e.\ their nodes are over $[n]^\star$.
\begin{definition}[Rabin tree automaton]
\label{def_Rabin_tree_automaton}
A \emph{Rabin tree automaton} over the alphabet $\Sigma$ has the form $\cA = (Q, q_0, \longrightarrow, \Omega)$ 
with a finite state set $Q$, initial state $q_0$, transition relation $\longrightarrow \subseteq Q \times \Sigma \times Q^n$,
and $\Omega = \{(L_1,U_1), \ldots, (L_m,U_m)\}$ is a collection of ``accepting pairs'' of state sets $L_i, U_i \subseteq Q$.
A \emph{run} of $\cA$ on a tree $T$ is a map $\rho : [n]^\star \rightarrow Q$ with $\rho(\varepsilon) = q_0$
and $(\rho(w),T(w),\rho(w1),\ldots,\rho(wn)) \in \longrightarrow$ for $w \in [n]^\star$.
For a path $\pi$ in $T$ and a run $\rho$ denote by 
$\mathrm{In}(\rho \mid \pi)$ the set of states that appear infinitely often in the restriction of $\rho$ to $\pi$.
A run $\rho$ of $\cA$ is \emph{successful} if 
\begin{center}
	for all paths $\pi$ there exists an $i \in [m]$ with \\
	$\mathrm{In}(\rho \mid \pi) \cap L_i = \emptyset$ and $\mathrm{In}(\rho \mid \pi) \cap U_i \neq \emptyset$.
\end{center}
A tree $T$ is \emph{accepted} by the Rabin tree automaton if some run of $\cA$ in $T$ is successful.
\end{definition}

\begin{theorem}[Rabin's Theorem, \cite{bk_rabin72}]
Any non-empty Rabin recognizable set of trees contains a regular tree.
\end{theorem}
Since condition ($\condsym$) is necessary and sufficient for the embeddability of regular framified constraint graphs, 
we get:
\begin{lemma}
\label{lem_LA_neq_iff_sol}
Let $\cA_{\emb}$ be a Rabin tree automaton that accepts exactly the consistent trees over $\Sigma_\fr$ satisfying $(\condsym)$.
There is an embeddable constraint graph iff $L(\cA_{\emb}) \neq \emptyset$.
\end{lemma}
\begin{proof}
If there is an embeddable constraint graph $\mathcal{G}$, then it has some framification $\cG_\fr$ (Observation~\ref{obs_framification}),
which satisfies the condition $(\condsym)$ (Lemma~\ref{lem_condition_necessary}). Therefore the tree representation of $\cG_\fr$ 
is accepted by $\cA_{\emb}$ and $L(\cA_{\emb}) \neq \emptyset$.
For the other direction, assume $L(\cA_{\emb}) \neq \emptyset$. 
Then by Rabin's Theorem, there is a regular tree $T \in L(\cA_{\emb})$, which satisfies the condition $(\condsym)$. 
By Lemma~\ref{lem_condition_suff_for_regular}, we have that the constraint graph represented by $T$ is embeddable.
\end{proof}

Therefore it remains to show that the condition ($\condsym$) is indeed
verifiable by a Rabin tree automaton. We do this next.

\paragraph{Checking consistency of trees}
In our constructions of automata, it is useful to assume that they run on
 trees over $\Sigma_\fr$ that are consistent (in the sense of
Definition~\ref{def_consistent_frames}), rather than complicating the
constructions by incorporating the consistency check. 
Therefore we first describe an automaton $\cA_{\mathrm{ct}}$ which accepts
exactly the consistent trees, which we later intersect with the appropriate
automata. $\cA_{\mathrm{ct}}$ simply verifies the 
conditions of Definition~\ref{def_consistent_frames} by only having transitions between
consistent pairs of frames, and making sure the root vertex is labeled with a frame whose vertex set consists exactly of $\mathsf{Reg}_{\cC,\cT}^{\mathrm{bot}}$. 
The comparison between subgraphs of pairs of frames requires $\cA_{\mathrm{ct}}$ to remember the previous letter, 
and therefore its state set is exponential in \size{\cC,\cT}.

\paragraph*{Verifying  $(\condsym)$ with a Rabin tree automaton}
We describe an automaton $\cB$ which runs on consistent trees over $\Sigma_\fr$, and
finds a pair of registers which violates 
($\condsym$).
The desired $\cA_{\mathrm{emb}}$ is the complement of $\cB$
intersected with $\cA_{\mathrm{ct}}$. 

We now define $\cB$ and describe its behavior. We let
\[ \mathcal{B} = (Q, q_0, \longrightarrow, (\emptyset, U)), \mbox{\qquad
  with:} \]

\begin{compactitem}
	\item $Q = \{q_0,q_1,q_2\} \cup Q_p$ where  $ Q_p$ is the set of
          \emph{path states}
	\[ Q_p = \sfReg_{\cC,\cT} \times \sfReg_{\cC,\cT} \times \{f,b \}
        \times \{ 0,1\}. \]
	\item
	We describe $\longrightarrow$ next. 
        The initial state $q_0$ and the state $q_2$ both
        represent that the problematic pair is (a) in the current subtree, (b) but not in the current node. 
From either of them, $\cB$ picks one child for which (a) is also true, and possibly also
(b). In the latter case, it moves to $q_2$ for that child, while
the other children go into $q_1$. State $q_1$ means that the problematic
pair is not in the subtree, and once $\cB$  visits some node in 
$q_1$, it stays in $q_1$ for all its descendants. 
          
	Denote by $\mathbf{q}^e_i$ the $n$-tuple containing $q_2$ for entry $i$
			and $q_1$ for every other entry.
	\begin{itemize}
		\item 
		For every $i \in [n]$ and $\sigma \in \Sigma_\fr$ we have 
		\[ q_0 \overset{\sigma}{\longrightarrow} \mathbf{q}^e_i
                \mbox{\qquad and \qquad}  q_2 \overset{\sigma}{\longrightarrow} \mathbf{q}^e_i\]
		\item
		For every $\sigma \in \Sigma_\fr$, we have
		$q_1 \overset{\sigma}{\longrightarrow} (q_1, \ldots, q_1)$ 
	\end{itemize}
        At some point, $\cB$ moves from a node where both (a) and (b) are true (that
is, $q_0$ or $q_2$) to a node where (b) no longer holds, i.e., it guesses that
the problematic pair is in that node $u$.
At this point, it guesses the problematic pair $x,y$
and whether it is the forward path $f$ or the backward path $b$ which will be
strict. This will be stored in the flag $f$ or $b$, which once chosen cannot change during the
run.

If the guessed pair $x,y$ has a $\leq$
relation (needed for the strict edge
from $f(0)$ to $b(0)$ required by ($\condsym$)), $\cB$
transitions accordingly to a path state $(x,y,f,0)$ or $(x,y,b,0)$;

For every $i \in [n]$, $h \in \{f,b\}$, and $- \in \{0,1\}$, denote by $(x, y, h , - )_i$ the $n$-tuple containing $(x, y, h , -)$ for entry $i$
		and $q_1$ for every other entry.
		\begin{itemize}
			\item
		          For every $i \in [n]$ and $h \in \{f,b\}$,
                          if $e_<(x^{\rbot},y^{\rbot}) \in \sigma$ we have 
		$q_0 \overset{\sigma}{\longrightarrow} (x,y,h,0)_i$
		and
		$q_2 \overset{\sigma}{\longrightarrow} (x,y,h,0)_i$
		\end{itemize}
                Then $\cB$ attempts to expand $f$ and $b$ by guessing a child
                $v$ and a new
                pair $z,w$ with a strict edge between $z$ and $w$. It moves 
                to the appropriate path state for the child $v$, and to $q_1$ for
                the remaining children.
                When doing so, is also uses another binary flag to indicate 
                whether $\cB$ just witnessed a 
strict edge relevant to $f$ or $b$ ($1$), or not ($0$).

We describe the transitions for the case where  the forward path
is strict; there are similar transitions for backward paths. 
 If the guess correctly   extends $f$ and $b$, that is,
                      \[ e_<(z^{\rbot},w^{\rbot}) \in \sigma \mbox{\qquad and
                        \qquad}  e_<(y^{\rtop},w^{\rbot}) \notin
                      \sigma \]
then, for every $i \in [n]$,
		\begin{itemize}
		\item
                  if the current edge  on the forward path is strict, that is, $e_<(x^{\rtop},z^{\rbot}) \in
                  \sigma$, we have 
 \[ (x,y,f,-) \overset{\sigma}{\longrightarrow} \mathbf{(} z, w, f
,1 \mathbf{)}_i \] 
		\item
		and if the current edge is not
                strict, that is, $e_=(x^{\rtop},z^{\rbot}) \in \sigma$,
                then we have \[ (x,y,f,-) \overset{\sigma}{\longrightarrow}
                \mathbf{(} z, w, f ,0 \mathbf{)}_i \]
             \end{itemize} 
\item  	$U = \{q_1\} \cup \{(x,y,h,1) \mid h \in \{f,b\}
  \}$.
  
  Paths looping in $q_1$ are successful, and to guarantee that the guessed path
  is strict, it must contain infinitely many strict edges (marked with flag $1$).
   \end{compactitem}

The number of states of $\cB$ is polynomial in \size{\cC,\cT}, and the alphabet is exponential. 
As mentioned before, the automaton $\cA_{\mathrm{emb}}$ is the complement automaton of $\mathcal{B}$ intersected with $\cA_{\mathrm{ct}}$. 
It has the same alphabet, but  
it may have exponentially many more states \cite{ar_MullerSchupp95}. 

\begin{proposition}
  There is a Rabin tree automaton $\cA_{\emb}$
  that accepts exactly the consistent trees over $\Sigma_\fr$ that satisfy
  $(\condsym)$,  whose number of states is bounded by a single exponential in
  $||\cC,\cT||$ and whose $\Omega$ has a constant number of pairs.
\end{proposition}

\subsection{Deciding satisfiability}
\label{subsec_putting_automata_together}

The automaton $\cA_{\emb}$ provides us an effective way to decide the 
embeddability of a constraint graph. With this central ingredient in place, we
are ready to put together an algorithm for checking the
satisfiability of  $\cC$ \wrt $\cT$. 
We do so by building an automaton  $\cA_{\cT,\cC}$ whose language is not empty iff
$\cC$ is satisfiable \wrt $\cT$.
In a nutshell, we obtain it by intersecting 
$\cA_{\emb}$  and an automaton for deciding
 satisfiability of the abstraction to $\ALCF$. For the latter, we may
rely on existing constructions from the literature.

\paragraph{Satisfiability of the abstracted $\ALCF$ part}
A well known construction of a looping automaton which accepts exactly the 
tree models of an $\ALC$ concept \wrt a TBox can be found in \cite{Baader09}.
Note that, for  completeness,
it is important that it accepts all tree models, as opposed to e.g. accepting
some canonical model which may not necessarily have an embeddable constraint
graph.
That construction can be easily adapted to obtain a Rabin tree automaton $\cA_{\mathrm{alcf}}$,
which also ensures functionality of the appropriate roles.
The automaton $\cA_{\mathrm{alcf}}$ runs on trees over the alphabet $\Xi$, which consists of 
sets of the concept names in $\cC,\cT$ and a single role name from $\cC,\cT$.
The role name in each letter indicates the role with which a logical element is connected to its parent.
The states of $\cA_{\mathrm{alcf}}$ are maximal consistent sets of the subexpressions in $\cC,\cT$, also known as Hintikka sets.
The number of states of $\cA_{\mathrm{alcf}}$ and the alphabet are
exponential in \size{\cC,\cT}, and its $\Omega$ has a constant number of pairs.

\paragraph{Pairing the alphabet}
The final automaton $\cA_{\cT,\cC}$ should accept only representations of
models of the abstraction of $\cC$ and $\cT$ whose
constraint graph is embeddable. Since one check is done by
$\cA_{\mathrm{alcf}}$ and the other by $\cA_{\emb}$, we modify both automata
to use the same alphabet. 
We let  $\cA'_{\emb}$  and $\cA'_{\mathrm{alcf}}$
be the modification of $\cA_{\emb}$  and $\cA_{\mathrm{alcf}}$ to trees over
the product alphabet $\Sigma_\fr \times \Xi$, while ignoring the irrelevant
part of each letter. Clearly, the state sets of   $\cA'_{\emb}$  and 
$\cA'_{\mathrm{alcf}}$  are not affected and remain exponential in \size{\cC,\cT},
nor are their $\Omega$ sets, which still have a constant number of pairs.

\paragraph{Matching the alphabets}
It is not enough to verify if a tree over  $\Sigma_\fr \times \Xi$
is accepted by $\cA'_{\emb}$, which ignores $\Xi$, and by 
$\cA'_{\mathrm{alcf}}$,  which ignores $\Sigma_\fr$: such a tree could just
pair a model of the abstraction with a totally unrelated constraint graph.
We need to verify that the constraint graph matches the interpretation of the
abstraction.
For this, we take an automaton $\cA_{\mathrm{m}}$ that considers both parts of the product alphabet $\Sigma_\fr \times \Xi$
and accepts the trees where the restriction of the input to $\Xi$ %%, corresponding to the abstraction to $\ALC$, 
induces the constraint graph corresponding to the restriction of the input to $\Sigma_\fr$.
This is done by verifying the conditions described in Definition~\ref{def_constraint_graph} 
while applying the placeholders in the $\Xi$ part of the letter to the $\rbot$ vertices in the $\Sigma_\fr$ part of the letter.
Such a test can be built into the transition relation, using a 
constant number of states.

\paragraph{Putting the automata together}
Finally, 
we build  $\cA_{\cT,\cC}$ as the intersection of  
$\cA'_{\mathrm{emb}}$, $\cA'_{\mathrm{alcf}}$, and $\cA_{\mathrm{m}}$.
Each tree it accepts represents a model of the abstraction of $\cC$ \wrt $\cT$ whose constraint
graph can be embedded into  $\bZ$, yielding the desired reduction of
satisfiability to automata emptiness.

\begin{proposition}
  There is a Rabin tree automaton $\cA_{\cT,\cC}$ whose state set is bounded
  by a single exponential in \size{\cC,\cT} and the number of pairs in its $\Omega$ is bounded by a polynomial 
	in \size{\cC,\cT}, such that
  $L(\cA_{\cT,\cC}) \neq \emptyset$ iff $\cC$ is satisfiable \wrt $\cT$.
   \end{proposition}

Since emptiness of  Rabin tree automata is decidable in time  polynomial in $Q$ and
exponential in the number of pairs in $\Omega$  \cite{EmersonJ88}, 
our main result follows.

\begin{theorem}
Satisfiability \wrt general TBoxes in $\ALCUP(\cZ_c)$ is decidable in \\ \EXPTIME.
\end{theorem}

This bound is tight: satisfiability \wrt general TBoxes is 
\EXPTIME-hard already for plain \ALC \cite{Schild91}.

\section{Beyond \texorpdfstring{$\ALCUP(\cZ_c)$}{ALCF-P(Zc)}}
\label{sec_other_pred}

In this section, we discuss some variants of our construction and how our
results extend to other closely related settings.

\paragraph{Undefined register values}
Classical concrete domains often allow for the predicate $\uparrow x$ which is interpreted as the register $x$ having undefined value.
In order to support this in our setting, we expand $\cZ_c$  to $\cZ_{c,\mathsf{und}}$ by 
adding a fresh element
to the integers to obtain $\mathbb{Z} \cup \{\sfu\}$, and adding a unary predicate $\mathsf{und}$ to the predicates of $\cZ_c$.
Our approach is adapted by redefining frames as follows.
The set $\mathbf{U}$ of labels also includes $U_{\mathsf{und}}$, and the first condition in Definition~\ref{def_frame}
is rephrased to be:
\begin{enumerate}
	\item 
	There is an edge between every pair of vertices which are not labeled $U_{\mathsf{und}}$.
\end{enumerate}
Note that due to condition~\ref{cond_equality}, also every pair of vertices labeled $U_{\mathsf{und}}$ is connected (with an equality edge).
%------------------------------
\paragraph{Adding $\sfint$ or $\mathsf{nat}$ predicates to dense domains}

When operating over a dense concrete domain such as the rationals or the reals, 
it can be useful to have a predicate which enforces that certain registers hold integer or natural number values.
We show that such predicates may be added to our setting while maintaining our complexity.

The predicate $\sfint(x)$ is interpreted as $\{u \in \Delta^{\cI} \mid \beta(u,x) \in \mathbb{Z}\}$,
and similarly for $\sfnat(x)$. Note that $\sfnat(x) \equiv \sfint(x) \sqcap \exists \epsilon. \lpath 0 \leq S^0 x \rpath$, so we limit our treatment to $\sfint$.

We describe how to add the $\sfint$ predicate to $\cR_c$, which is $\cZ_c$ with the reals as the domain, while maintaining our complexity bounds.
We need to check whether the subgraph induced by the registers satisfying the $\sfint$ predicate is embeddable, which boils down to
making sure there is no pair of registers 
with infinitely many $\sfint$ registers between them that must have different values.
Therefore we adapt the automaton $\cB$ to look for a pair of registers violating the following updated condition:
\begin{description}\item[$(\condsym_\sfint)$]
There are no $(u,x),(u,y) \in \Delta \times \mathsf{Reg}_{\cC,\cT}$ in $\mathcal{G}_{\mathrm{fr}}$ 
for which we have that: there exists an infinite $\mathbf{w} \in [n]^\omega$ and
\begin{enumerate}
	\item 
	an infinite forward path $f$ from $(u,x)$ along $\mathbf{w}$
	\item
	an infinite backward path $b$ from $(u,v)$ along $\mathbf{w}$
\end{enumerate}
such that $f$ or $b$ is strict and has infinitely many $\sfint$ labels, 
and such that for every $i \in \mathbb{N}$, there is a strict edge from $f(i)$ to $b(i)$.
\end{description}
The automaton $\cB$ is adapted so that it guesses whether $f$ or $b$ is strict and has infinitely many $\sfint$ registers, 
and we add to the acceptance condition the requirement that it witnesses
infinitely many $\sfint$ registers on the path it guessed.
In the proofs we also  redefine the notion of strict length of paths,
counting only 
 $\sfint$ registers that occur between strict edges.

\subsection{The logic \texorpdfstring{$\ALCP(\cD)$}{ALCP(D)}}

The DL $\ALCP(\cD)$ was introduced for a general domain $\cD$ in
\cite{DBLP:conf/cade/Lutz01}, and a related DL was studied already 
in \cite{DBLP:conf/ki/BaaderH92}.
While  most extensions of DLs with concrete domains allow only functional
roles on the paths participating in the concepts that refer to the concrete
domain,  $\ALCP(\cZ_c)$ allows arbitrary roles.
It is similar to our logic, but it can compare values on  different paths,
while \alcupz can only compare values on the same path.

We briefly recall the definition of $\ALCP(\cZ_c)$, and refer to
\cite{DBLP:conf/cade/Lutz01} for details.
Since the syntax of $\ALCP$ does not allow
Boolean combinations of constraints, we enrich $\cZ_c$
to 
explicitly include $\neq, \leq,$ and $\geq$.
This provides a closer comparison between the logics, and in the case of
$\ALCUP$, it is equivalent to the simpler $\cZ_c$ considered so far.

$\ALCP(\cZ_c)$ is the augmentation of $\mathcal{ALCF}$ with\footnote{$\ALCP(\cZ_c)$
  supports $x \uparrow$, which we can simulate as above.}
\begin{itemize}
	\item
	$\exists P x. = c$ and $\forall P x. = c$ where $c \in \bZ$ and $P$ is a sequence of role names and $x$ a register name, and similarly for $\neq c$.
	The formulas apply the constraint to the $x$ register of the last element on a $P$ path.
	\item 
	$\exists P_1 x_1, P_2 x_2. \theta$ and $\forall P_1 x_1, P_2 x_2. \theta$ where $\theta \in \{\leq,< , =, \neq, >, \geq\}$
	and each $P_i x_i$ is a sequence of role names followed by a register name.
	The formulas apply the constraint to the $x_1$ register of the last element on a $P_1$ path and the $x_2$ register of the last element on a $P_2$ path,
	where both paths start at a common element.
\end{itemize}

We now translate $\ALCP(\cZ_c)$ to $\ALCUP(\cZ_c)$. In most cases  we use additional
registers. 
\begin{itemize}
	\item
	$\exists P x. = c$ translates to $\exists P. \lpath S^{|P|} x = c \rpath$ and 
	$\forall P x. = c$ translates to $\forall P. \lpath S^{|P|} x = c \rpath$.
	\item For existential concepts $\exists P_1 x_1, P_2 x_2.{\theta}$, an
          easy translation is possible by using two fresh register names  $copy\text{-}g_1$ and
          $copy\text{-}g_2$, which  intuitively store the values at the end of $P_1$ and $P_2$. For
          example,  $\exists P_1 x_1,P_2 x_2. <$ translates to
		\begin{align*}
		C := & \, 
		\exists  P_1.\lpath S^{0} copy\text{-}x_1 = S^{|P_1|} x_1\rpath \\
		\sqcap & \, \exists P_2.\lpath S^{0} copy\text{-}x_2 = S^{|P_2|} x_2\rpath \\
		\sqcap & \, \exists \epsilon. \lpath S^0 copy\text{-}x_1 < S^0 copy\text{-}x_2\rpath
		\end{align*} 
                The translations for $\theta \in \{\leq,< , =, \neq, >, \geq\}$ are similar.

              \item In the cases of $\forall P_1 x_1, P_2 x_2.\theta$, we
                treat differently the paths of only functional roles,
                 and the case where arbitrary
roles may occur. 
\begin{itemize}
      \item   If all roles occurring in $P_1$ and $P_2$ are functional, then an
        encoding similar to  $\exists P_1 x_1, P_2 x_2.\theta$ can be
        used. For example,  $\forall P_1 x_1,P_2 x_2.{<}$ translates to 
		$\neg\exists P_1. \top \sqcup \neg \exists P_2. \top \sqcup
                C$, where  $C$ is as above.
              \item If non-functional roles occur in $P_1$ and $P_2$, then we
                may need to compare numbers on several paths, and we may need
                more sophisticated tricks. 
                For $\theta \in \{\leq,< , =, \neq, >, \geq\}$, this is still possible
                using just a few registers. 
                  For example, we can translate 
                 $\forall P_1 x_1,P_2 x_2. <$ as 
		\begin{align*}
		\neg \exists P_1. \top \sqcup 
		& \big( 
		\exists P_1. \lpath S^{|P_1|} x_1 = S^0 copy\text{-}x_1 \rpath  \\
		& \sqcap
		\forall P_1. \lpath S^{|P_1|} x_1 \leq S^0 copy\text{-}x_1 \rpath  \\
		& \sqcap
		\forall P_2. \lpath S^{|P_2|} x_2 > S^0 copy\text{-}x_1 \rpath 
		\big)
		\end{align*}
		We are essentially ensuring, via $copy\text{-}x_1$, that the largest value of $x_1$ seen with a $P_1$
		path is smaller than every value of $x_2$ seen with a $P_2$ path.

		If $\theta$ is $\neq$, our translation requires
                exponentially many new register names. 
		Given $\cC$, $\cT$
		 we can ascertain a degree $k$ of some
               tree model (if any model exists). In this model there would be at most 
		$|P_1|^k$ different values to consider for the satisfaction of the constraint. 
		Slightly abusing notation, we express that the $x_1$ registers at the end of $P_1$ paths 
		contain values from a finite set which appears in the fresh register names of the common ancestor, 
		and that this set does not intersect with the set of values of the $x_2$ registers at the end of $P_2$ paths:
		\begin{align*}
		&\exists \epsilon. \lpath S^0 x_1 \neq \cdots \neq S^0 x_{|P_1|^k} \rpath \\
		\sqcap \,
		&\forall P_1. \lpath S^{|P_1|} x_1  = x_1 \vee \cdots \vee S^{|P_1|} x_1  = x_{|P_1|^k} \rpath \\
		\sqcap \,
		&\forall P_2. \lpath S^{|P_2|} x_2  = x_1 \vee \cdots \vee S^{|P_2|} x_2  \neq x_{|P_1|^k} \rpath 
		\end{align*} 
       \end{itemize}
\end{itemize}

This translation allows us to give an upper bound on the complexity of
reasoning \wrt general TBoxes  in the DL $\ALCP(\cZ_c)$, which to the best
of our knowledge, had never been provided before.
Our upper bounds also apply if we replace the integers by the real numbers,
with or without $\sfint$ and $\mathsf{nat}$ predicates in the concrete
domain. 

\begin{theorem}
Satisfiability \wrt general TBoxes in $\ALCP(\cZ_c)$ is decidable in
2ExpTime, and it is ExpTime-complete if there is a constant bound on the
length of any path $P_1$ that contains non-functional roles and
occurs in a concept of the form $\forall P_1 x_1,P_2 x_2.{\neq}$.
\end{theorem}

\section{Conclusions}

We have closed a long-standing open question in the  literature of DLs with
concrete domains: reasoning with general TBoxes in \ALC extended
with the non-dense domain $\cZ_c$ is \EXPTIME-complete, and hence not harder than in plain \ALC, even if
arbitrary paths of (not necessarily functional) roles are allowed to refer to
the concrete domain. This positive result extends to other domains that have
been advocated for in the literature, for example, comparisons over the real or
rational numbers  but with the  $\sfint$ and $\mathsf{nat}$ predicates.
Our technique builds on ideas used for constraint LTL in \cite{DemriD07}, and our condition $(\condsym)$ is very similar to the condition used in
that paper. Lifting the results from linear structures, as in
LTL,  to the tree-shaped ones needed in \ALC is not trivial.
It remains an open question  whether our technical results can be transferred to fragments of
constraint CTL$^*$ to obtain new complexity bounds. 
Natural next steps are exploring other DLs, for 
example $\mathcal{SHIQ}^{\cP}(\cZ)$,  and considering ABoxes and 
instance queries.

\section*{Acknowledgments}
This work was supported by the Austrian Science Fund (FWF) projects P30360, P30873, and W1255.

\bibliographystyle{plain}
\bibliography{ref}

\newpage
\appendix
\newcommand{\pathpic}{

\begin{tikzpicture}

\node(v1) at (-2,4) {$v$};
\node(v1z1) at (0,4) {$z$};
\node(v1z2) at (1,4) {$z'$};
\node (u1) at (0,3) {};
\node (u2) at (0,2) {};
\node(v2) at (-2,1) {$v' = vw$};
\node(v2z1) at (0,1) {$z$};
\node(v2z2) at (1,1) {$z'$};
\node(v3) at (-2,-2) {$v'' = vww$};
\node(v3z1) at (0,-2) {$z$};
\node(v3z2) at (1,-2) {$z'$};

\tikzset{decoration={snake, amplitude=.3mm}}

\draw[color=blue,decorate]  (v1z1) -- (v1z2);
\draw[color=blue,decorate]  (v1z1) -- (u1);
\draw[color=blue,decorate]  (u2) -- (v2z1);
\draw[color=blue,->] (u1) -- (u2);
\node[color=blue] at (-0.5,2.5) {$f_1$};

\draw[decorate]  (v2z1) -- (v2z2);
\draw[color=red,decorate]  (v2z2) -- (v1z2);
\node[color=red] at (1.5,2.5) {$b_1$};

\node(u3) at (0,0) {};
\node(u4) at (0,-1) {};

\draw[decorate]  (v2z1) -- (u3);
\draw[decorate]  (u4) -- (v3z1);
\draw[->] (u3) -- (u4);

\draw[decorate]  (v3z1) -- (v3z2);
\draw[decorate]  (v3z2) -- (v2z2);

\end{tikzpicture}
}

\section{Proof for the atomic normal form in Section~\ref{subsec_normalization}}

\subsection*{Proof of Lemma~\ref{lem_new_normal_form}}

First we demonstrate how negation may be removed from atomic constraints using generic examples:
\begin{itemize}
	\item 
	$\neg (x = y)$ can be rewritten as $(x < y) \vee (y < x)$
	\item
	$\neg (x = 0)$ can be rewritten using a fresh register name $z$ as $(z = 0) \wedge ((x < z) \vee (z < x))$
\end{itemize}

  Let $\cC'$ and $\cT'$ be a concept and a TBox in
  $\ALCUP(\cZ_c)$ that are negation free. Let $W_{\mathsf{roles}}$, $W_{\mathsf{reg}}$, and $W_{\mathsf{paths}}$ be the
  sets of role names, register names, and role paths that appear in $\cC'$ and
  $\cT'$, respectively. Let $d$ be the maximal depth of path
  constraints used in $\cC'$ and $\cT'$.

The proof is split into three parts; In the first part, we propagate the original register values
into copy-registers which will make them available locally.
In the second part, we use fresh ``test" concept names to indicate how the atomic values relate to one another, essentially acting 
as the logical connectives.
Finally, we put it together by rewriting the original concept and TBox into atomic normal form.

\paragraph{Part I} 
In the first step, by relying on the tree model property, we copy in each node $u$ the
registers of the ancestors that may occur in the constraints with the
registers of $u$ by propagating the values one step at a time.  
Assume that $W_{\mathsf{reg}}= \{x_1^{0},\ldots,x_m^{0}\}$. 
For every $i$ where  $1\leq  i \leq m$, every $k$ where $1\leq  k\leq d  $ and every $P \in W_{\mathsf{paths}}$, 
we take a fresh register name $x_{i,P}^k$ which will serve as a copy-register.  
We create a TBox $\cT_{\mathsf{prop},d}$ as follows:
\begin{align*}
\cT_{\mathsf{prop},d} = \big\{\top \sqsubseteq \forall r.
  \lpath S^1 x_{i}^{0} = S^0 x_{i,P}^{1} \rpath \mid r\in
  W_{\mathsf{roles}}, 1\leq i \leq m, P \in W_{\mathsf{paths}} \big\} \\
	\cup \big\{\top \sqsubseteq \forall r.
  \lpath S^1 x_{i,P}^{k} = S^0 x_{i,P}^{k-1} \rpath \mid r\in
  W_{\mathsf{roles}}, 1\leq i \leq m, 2\leq k\leq d, P \in W_{\mathsf{paths}} \big\}
\end{align*}
Note that along every path $P$, the TBox $\cT_{\mathsf{prop},d}$ propagates values into copy-registers associated with all the paths in $W_{\mathsf{paths}}$, not
just into the copy-registers associated with $P$. We will later restrict our attention to the appropriate copy-registers depending on context.
The next claim follows with a straightforward inductive construction:
\begin{claim}
\label{cl_part_I_app}
Every tree model of $\cC'$ w.r.t.\
$\cT' \cup \cT_{\mathsf{prop},d}$ contains a tree model of $\cC'$
w.r.t.\ $\cT'$, and every tree model of $\cC'$ w.r.t.\ $\cT'$
can be expanded to a tree model of $\cC'$ w.r.t.\
$\cT' \cup \cT_{\mathsf{prop},d}$.
\end{claim}	
\begin{proof}
We inductively describe an expansion of a tree model $\cI$ of $\cC'$ w.r.t.\ $\cT'$ 
such that the final expansion is a tree model of $\cC'$ w.r.t.\
$\cT' \cup \cT_{\mathsf{prop}}$. 
We will simply copy the values in the original registers into their corresponding copy-registers.
For copy-registers of elements that are at a smaller depth than the associated path $P$,
we will assign an arbitrary value (namely $0$).

\begin{enumerate}
	\item 
	We first describe an expansion $\cJ_1$ of $\cI$ that will model $\cC'$ w.r.t.\
	$\cT' \cup \cT_{\mathsf{prop},1}$.
	For the root element $\varepsilon$, for every $i$ where $1 \leq i \leq m$, and for every $P \in W_{\mathsf{paths}}$, set
	\[
	(\varepsilon,x_{i,P}^{1})^{\cJ_{1}} = 0.
	\]
	For elements $u,v \in \Delta$ where $u$ is the parent of $v$,
	for every $i$ where $1 \leq i \leq m$, and for every $P \in W_{\mathsf{paths}}$, set 
	\[
	(v,x_{i,P}^{1})^{\cJ_{1}} = (u,x_{i}^{0})^{\cI}. 
	\]
	We have that the copy-registers $\{x_{i,P}^1 \mid 1 \leq i \leq m, P \in W_{\mathsf{paths}} \}$ are defined for all elements,
	and the newly assigned register values $\cJ_1$ satisfy the axioms in $\cT_{\mathsf{prop},1}$.
	\item
	We now describe an expansion $\cJ_{d'}$ given a tree model $\cJ_{d'-1}$ of $\cC'$ w.r.t.\ $\cT' \cup \cT_{\mathsf{prop},d'-1}$, 
	where the copy-registers 
	\[
	\{ x_{i,P}^{k} \mid 1 \leq i \leq m, P \in W_{\mathsf{paths}}, 1 \leq k \leq d'-1 \}
	\]
	are defined for all elements.
	For the root element $\varepsilon$, for every $i$ where $1 \leq i \leq m$, and for every $P \in W_{\mathsf{paths}}$, set
	\[
	(\varepsilon,x_{i,P}^{d'})^{\cJ_{d'}} = 0.
	\]
	For elements $u,v \in \Delta$ where $u$ is the parent of $v$, and
	for every $i$ where $1 \leq i \leq m$, and for every $P \in W_{\mathsf{paths}}$, set 
	\[
	(v,x_{i,P}^{d'})^{\cJ_{d'}} = (u,x_{i,P}^{d'-1})^{\cJ_{d'-1}}.
	\]

	We show that $\cJ_{d'}$ is a tree model of $\cC'$ w.r.t.\ $\cT' \cup \cT_{\mathsf{prop},d'}$.
	The newly assigned register values satisfy the axioms in $\cT_{\mathsf{prop},d'} \setminus \cT_{\mathsf{prop},d'-1}$, 
	and $\cJ_{d'-1}$ satisfies $\cC'$ w.r.t.\ $\cT' \cup \cT_{\mathsf{prop}, d'-1}$. 
	Since the expansion does not alter previously defined values, and since 
	the register names $x_{1,P}^{d'}, \ldots, x_{m,P}^{d'}$ for $P \in W_{\mathsf{paths}}$ 
	do not appear in neither $\cC'$ nor $\cT' \cup \cT_{\mathsf{prop},d'-1}$, we 
	have that $\cJ_{d'}$ is a tree model of $\cC'$ w.r.t.\ $\cT' \cup \cT_{\mathsf{prop},d'}$.
\end{enumerate}
 
Hence $\cJ_d$ is a tree shaped expansion of $\cI$ which satisfies $\cC'$
w.r.t.\ $\cT' \cup \cT_{\mathsf{prop},d}$.

\end{proof}	
We write $\cT_{\mathsf{prop}}$ for $\cT_{\mathsf{prop},d}$ from now on.

\paragraph{Part II} 
In this step, we create some ``test'' concept names and axioms
that will allow to check whether a given constraint is satisfied in a
certain path in a tree model. 
For $P \in W_{\mathsf{paths}}$ and an atomic constraint $\Theta$, let $\mathsf{loc}(\Theta,P)$ denote the constraint obtained from $\Theta$
by replacing each occurrence of $S^j x_{i}^{0}$ with $ S^0 x_{i,P}^{|P|-j}$. 
I.e.\ a reference to an original register at a large depth is replaced with a local reference to its copy-register.

Denote by $W_{\mathsf{cnstr}}$ the (sub)constraints that appears in $\cC'$ or $\cT'$.
For each $P \in W_{\mathsf{paths}}$ and each $\Theta \in W_{\mathsf{cnstr}}$, 
take a fresh concept name
$T_{P,\Theta}$. For each such $P$ and $\Theta$ we add to a TBox $\cT_{\mathsf{loc}}$ the following axioms
\begin{enumerate}[label={({\bf A}\textsubscript{\arabic*}}),ref=({\bf A}\textsubscript{\arabic*}),series=conditions]
\item 
\label{it_and}
$T_{P,\Theta} \equiv T_{P,\Theta_1} \sqcap T_{P,\Theta_2}$ if $\Theta = \Theta_1\land  \Theta_2$ 
\item 
\label{it_or}
$T_{P,\Theta} \equiv T_{P,\Theta_1} \sqcup T_{P,\Theta_1}$  if $\Theta = \Theta_1\lor \Theta_2$ 
\item 
\label{it_loc}
$T_{P,\Theta} \equiv \exists \epsilon \lpath \mathsf{loc}(\Theta,P) \rpath$ if  $\Theta$ is an atomic constraint.
\end{enumerate}

We first show that the tree models we are interested in can be expanded along with these axioms:
\begin{claim}
\label{cl_part_II_app}
Every tree model of $\cC'$ w.r.t.\ $\cT' \cup \cT_{\mathsf{prop}}$ can be expanded to a tree model of $\cC'$ w.r.t.\
$\cT' \cup \cT_{\mathsf{prop}} \cup \cT_{\mathsf{loc}}$, and every tree model of 
$\cC'$ w.r.t.\
$\cT' \cup \cT_{\mathsf{prop}} \cup \cT_{\mathsf{loc}}$
is a tree model of $\cC'$ w.r.t\ $\cT' \cup \cT_{\mathsf{prop}}$.
\end{claim}
\begin{proof}

Let $\cI$ be a tree model of $\cC'$ w.r.t.\
$\cT' \cup \cT_{\mathsf{prop}}$. 
Let $h$ be the largest circuit-depth of a constraint 
$\Theta$ appearing in $\cT'$ or $\cT' \cup \cT_{\mathsf{prop}}$.
We inductively define $\cJ^h$ as an expansion of 
$\cI$ by interpreting the fresh concept names of the form $T_{P,\Theta}$.

\begin{enumerate}
	\item 
	We first describe $\cJ^0$ by interpreting $T_{P,\Theta}$ for atomic $\Theta$ and $P \in W_{\mathsf{paths}}$

	For $e \in \Delta$, we have 
	$e \in T_{P,\Theta}^{\cJ^0}$ if and only if 
	$\cI,(e) \models \mathsf{loc}(\Theta,P)$. That is, if and only if the copy-registers of $e$ satisfy the localized version of 
	$\Theta$. Note that in $\cJ^0$, elements may be labeled with $T_{P,\Theta}$ even if they are not the endpoint of a $P$-path (or even if they are not on 
	a $P$-path at all).

	We have that the axioms of the form in item~\ref{it_loc}, which are the only ones relevant in this case, are satisfied by the construction. 
	\item
	Let $\Theta_1,\Theta_2$ be such that $T_{\Theta_1,P}$ and $T_{\Theta_2,P}$ were interpreted in $\cJ^{h'-1}$.
	\begin{itemize}
		\item 
		If $\Theta = \Theta_1 \wedge \Theta_2$ then $e \in T_{\Theta,P}^{\cJ^{h'}}$ if and only if $e \in T_{\Theta_1,P}^{\cJ^{h'-1}} \cap T_{\Theta_2,P}^{\cJ^{h'-1}}$
		\item
		If $\Theta = \Theta_1 \vee \Theta_2$ then $e \in T_{\Theta,P}^{\cJ^{h'}}$ if and only if $e \in T_{\Theta_1,P}^{\cJ^{h'-1}} \cup T_{\Theta_2,P}^{\cJ^{h'-1}}$
	\end{itemize}
	The axioms of $\cT_{\mathsf{loc}}$ of the forms in items~\ref{it_and} and~\ref{it_or}, 
	which are the only ones relevant in this case, are satisfied by the semantics of the connectives $\wedge$ and $\vee$.
\end{enumerate}
Therefore we have that $\cJ^h$ is a tree-shaped expansion of 
$\cI$ that models $\cC'$ w.r.t.\
$\cT' \cup \cT_{\mathsf{prop}} \cup \cT_{\mathsf{loc}}$.

\end{proof}

Next, we show that $\cT_{\mathsf{prop}} \cup \cT_{\mathsf{loc}}$ indeed relate the satisfaction of constraints along path
to the test concept names. 
\begin{claim}
\label{cl_part_II_a_app}
Let $\cJ$ be a tree model of $\cT_{\mathsf{prop}} \cup \cT_{\mathsf{loc}}$, 
and let $P$ be a role path and $\Theta$ a constraint appearing in $\cC'$ or $\cT'$.
Then it holds that
\begin{enumerate}
\item if $e\in T_{P,\Theta}^{\cJ}$ and  $\cJ$ contains 
	a $P$-path
  $e_0,\ldots, e_{|P|}$ that ends at $e$ ($e=e_{|P|}$), then the path
  $e_0,\ldots, e_{|P|}$ satisfies the constraint $\Theta$ in $\cJ$;
\item if $\cJ$ has a $P$-path  $e_0,\ldots, e_{|P|}$ that satisfies $\Theta$, then $e_{|P|} \in T_{P,\Theta}^{\cJ}$.
\end{enumerate}
\end{claim}
Notice the qualification in item 1., as a $P$-path might not exist closer to the root in a tree model.
\begin{proof}
First item: let $e_0,\ldots, e_{|P|}$ be a $P$-path and let  $e_{|P|} \in T_{P,\Theta}^{\cJ}$. 
Note that from the axioms in $\cT_{\mathsf{prop}}$ we have that 
$(e_j,x_{i,P}^{0})^{\cJ} = (e_j, x_i^{{|P|}-j})^{\cJ}$ (this is actually true for every $P' \in W_{\mathsf{paths}}$). 
We proceed by induction on $\Theta$.
\begin{itemize}
	\item 
	If $\Theta$ is atomic, then by the axioms from item~\ref{it_loc} we have $\cJ, (e_{|P|}) \models \mathsf{loc}(\Theta,P)$.
	From the fact that $(e_j,x_i^{0})^{\cJ} = (e_j, x_{i,P}^{{|P|}-j})^{\cJ}$, 
	together with the definition of $\mathsf{loc}$ we get that $\cJ, (e_0,\ldots,e_{|P|}) \models \Theta$.
	\item
	Let $\Theta_1$ and $\Theta_2$ be constraints for which the claim holds.
	\item
	If $\Theta = \Theta_1 \wedge \Theta_2$, then from the axioms in item~\ref{it_and} we have that 
	$e_{|P|} \in T_{P,\Theta_1} \sqcap T_{P,\Theta_2}$. From the IH we have that $\cJ, (e_0,\ldots,e_{|P|}) \models \Theta_1$ 
	and $\cJ, (e_0,\ldots,e_{|P|}) \models \Theta_2$
	and the claim follows.
	\item
	The $\vee$ case follows similarly.
\end{itemize}

Second item: let a $P$-path $e_0,\ldots, e_{|P|}$ in $\cJ$ satisfy $\Theta$. 
Note that from the axioms in $\cT_{\mathsf{prop}}$ we have that $(e_j,x_i^{0})^{\cJ} = (e_j, x_{i,P}^{|P|-j})^{\cJ}$.
We proceed by induction on $\Theta$.
\begin{itemize}
	\item 
	If $\Theta$ is atomic, then from the fact that $(e_j,x_i^{0})^{\cJ} = (e_j, x_{i,P}^{|P|-j})^{\cJ}$, 
	together with the definition of $\mathsf{loc}$ we get that
	$\cJ, (e_{|P|}) \models \mathsf{loc}(\Theta,P)$ hence $e_{|P|} \in (\exists \epsilon. \lpath \mathsf{loc}(\Theta,P) \rpath)^{\cJ} $ 
	and from the axioms in item~\ref{it_loc} we get that $e_{|P|} \in T_{P,\Theta}^{\cJ}$.
	\item
	Let the claim hold for $\Theta_1$ and $\Theta_2$.
	\item
	If $\Theta = \Theta_1 \wedge \Theta_2$, then by the IH we have that $e_{|P|}$ is in $T_{P,\Theta_1}^{\cJ}$ and $T_{P,\Theta_2}^{\cJ}$, hence by
	semantics of $\ALCUP(\cZ_c)$ we have that $e_{|P|} \in T_{P,\Theta_1}^{\cJ} \sqcap T_{P,\Theta_2}^{\cJ}$ and by the axioms in item~\ref{it_and}
	we have that $e_{|P|} \in T_{P,\Theta}^{\cJ}$.
	\item
	The $\vee$ case follows similarly.
\end{itemize}
\end{proof}

\paragraph{Part III} 
In this final step, we use the locally available copy-registers and test concept names to rewrite $\cC'$
and $\cT'$ into $\cC$ and $\cT$ in atomic normal form, and use the previously proved claims 
to show equisatifiability.

Given a concept $D$ and a (possibly empty) role path
$P= r_1\cdots r_n$, we write $\exists P.D $ meaning 
\begin{enumerate}
	\item 
	the concept
	$\exists r_1 (\exists r_2(\cdots (\exists r_n.D)\cdots))$ when $n>0$,
	and
	\item
	the concept $D$ when $n=0$.
\end{enumerate}
The same notion is defined for
$\forall P.D $ in the obvious way. 

Let $\cC$ and $\cT^{*}$ be
obtained from $\cC'$ and $\cT'$, respectively, by replacing every concept
$\exists P.\lpath\Theta \rpath $ by $\exists P.T_{P,\Theta} $ and
every $\forall P.\lpath\Theta \rpath $ by $\forall P.T_{P,\Theta}
$. Our desired normalization is the concept $\cC$ equipped with the TBox
$\cT=\cT^{*}\cup \cT_{\mathsf{prop}} \cup \cT_{\mathsf{loc}}$.

Let $\cI$ be a tree model of $\cC'$ and $\cT'$.
By composing Claim~\ref{cl_part_I_app} and Claim~\ref{cl_part_II_app}, we get a tree model $\cJ$ of $\cC'$ w.r.t.\
$\cT' \cup \cT_{\mathsf{prop}} \cup \cT_{\mathsf{loc}}$, to which Claim~\ref{cl_part_II_a_app} applies.
We show that $\cJ$ is also a tree model of $\cC$ w.r.t.\ $\cT$ by showing that
$(\exists P.\lpath\Theta \rpath)^\cJ  = (\exists P.T_{P,\Theta})^\cJ$
(showing that $(\forall P.\lpath\Theta \rpath)^\cJ  = (\forall P.T_{P,\Theta})^\cJ$ is similar). 
\begin{itemize}
	\item 
	Let $e_0 \in (\exists P.\lpath\Theta \rpath)^\cJ$. Then there is a $P$-path $\vec{e} = (e_0, \ldots, e_{|P|})$ in $\cJ$ such that
	$\cJ, \vec{e} \models \Theta$, therefore by item 2 in Claim~\ref{cl_part_II_a_app} we have that $e_{|P|} \in T_{\Theta,P}^\cJ$,
	implying that $e_0 \in (\exists P. T_{\Theta,P})^\cJ$.
	\item
	Let $e_0 \in (\exists P. T_{\Theta,P})^\cJ$. Then there exists  a $P$-path $\vec{e} = (e_0, \ldots, e_{|P|})$ in $\cJ$ such that
	$e_{|P|} \in T_{\Theta,P}^\cJ$. By item 1 of Claim~\ref{cl_part_II_a_app}, we have that $\cJ, \vec{e} \models \Theta$ and therefore
	$e_0 \in (\exists P. \lpath \Theta \rpath)^\cJ$.
\end{itemize}
Therefore a tree model $\cI$ of $\cC'$ w.r.t.\ $\cT'$ can be expanded to a tree model of $\hat{C}$ w.r.t.\ $\hat{T}$.

Now we show that every tree model $\hat{\cJ}$ of $\cC$ w.r.t.\ $\cT$ is also a tree model of $\cC'$ w.r.t.\ $\cT'$.
Since $\cT \subseteq \cT_{\mathsf{prop}} \cup \cT_{\mathsf{loc}}$, Claim~\ref{cl_part_II_a_app} again applies to $\hat{\cJ}$. 
Like before, we have that $(\exists P.\lpath\Theta \rpath)^{\hat{\cJ}}  = (\exists P.T_{P,\Theta})^{\hat{\cJ}}$ and
$(\forall P.\lpath\Theta \rpath)^{\hat{\cJ}}  = (\forall P.T_{P,\Theta})^{\hat{\cJ}}$, therefore $\hat{\cJ}$ is a tree model of $\cC'$ w.r.t.\  
$\cT' \cup \cT_{\mathsf{prop}} \cup \cT_{\mathsf{loc}}$ (and in particular w.r.t.\ $\cT'$).

\subsection{Applying the ANF transformation to Example~\ref{ex_constraint}}

Here we provide an ANF transformation of a $\ALCUP(\cZ_c)$ concept and TBox based on the interpretation in Example~\ref{ex_constraint}.
First, let us name the concepts and constraints:
\begin{itemize}
	\item 
	$C_1$ denotes $\exists \epsilon. \lpath \Theta_1 \rpath$ where $\Theta_1$ is $S^0 x < S^0 y$
	\item
	$C_2$ denotes $\exists r. \lpath \Theta_2 \rpath$ where $\Theta_2$ is $S^1 x < S^0 y$
	\item
	$C_3$ denotes $\exists r. \lpath \Theta_3 \rpath$ where $\Theta_3$ is $\Theta_{31} \wedge \Theta_{32}$, and
	$\Theta_{31}$ is $S^0 x < S^1 x$ and $\Theta_{32}$ is $S^0 y = S^1 y$
	\item
	$C_4$ denotes $\exists r. \lpath \Theta_4 \rpath$
	where $\Theta_4$ is $\Theta_{41} \wedge \Theta_{42}$, and $\Theta_{41}$ is $S^0 x < S^1 x$ and $\Theta_{42}$ is $S^0 y < S^1 x$
\end{itemize}

Then we may say that the interpretation satisfies the concept
\[
C_2 \sqcap \exists r. (C_2 \sqcap C_4) \sqcap C_3 \sqcap \exists r. C_3
\]
w.r.t.\ the TBox $\cT = \{\top \sqsubseteq C_1\}$.

Note that the only path appearing in $C$ or $\cT$ is $r$, and that $C_1$ is already in ANF. 
We skip the construction of $\cT_{\mathsf{prop}}$, and assume that copies of parent register are available in $x_r^1$ and $y^1_r$.

Next, by introducing test concept name of the form $T_{r,\Theta}$ we construct $\cT_{\mathsf{loc}}$, which contains:
\begin{align*}
T_{r,\Theta_2} &\equiv \exists \epsilon. \lpath S^0 x^1_r < S^0 y \rpath
\\
T_{r,\Theta_3} &\equiv T_{r,\Theta_{31}} \sqcap T_{r,\Theta_{32}}
\\
T_{r,\Theta_{31}} &\equiv \exists \epsilon. \lpath S^0 x < S^0 x^1_r \rpath
\\
T_{r,\Theta_{32}} &\equiv \exists \epsilon. \lpath S^0 y = S^0 y^1_r \rpath
\\
T_{r,\Theta_4} &\equiv T_{r,\Theta_{41}} \sqcap T_{r,\Theta_{42}}
\\
T_{r,\Theta_{41}} &\equiv \exists \epsilon. \lpath S^0 x < S^0 x^1_r \rpath
\\
T_{r,\Theta_{42}} &\equiv \exists \epsilon. \lpath S^0 y < S^0 x^1_r \rpath
\end{align*}

Finally, by replacing the original $C_2,C_3,C_4$ with their test concept counterparts, we obtain the concept
\[
T_{r,\Theta_2} \sqcap \exists r. (T_{r,\Theta_2} \sqcap (\exists r. T_{r,\Theta_4})) \sqcap \exists r. T_{r,\Theta_3} \sqcap \exists r. (\exists r. T_{r,\Theta_3})
\]
and the new TBox $\cT \cup \cT_{\mathsf{loc}} \cup \cT_{\mathsf{prop}}$. Note that since $C_1$ was already in ANF, the original TBox $\cT$ does not change before being added to the final TBox.

\section{Proofs for the embeddability condition in Subsection~\ref{subsec_embd_cond}}

\subsubsection*{Proof of Lemma~\ref{lem_no_cycles_in_framification}}

We prove the lemma by contradiction. 
Let $p$ be a strict cycle in $\cG_\fr$ which spans vertices of exactly $k$ logical elements $u_1, \ldots, u_k$, and assume w.l.o.g.\
that $p$ starts and ends at $u_1$, and that $u_i$ is the parent of $u_{i+1}$ for $i \in [k-1]$.
If $k \leq 2$, then $p$ is a strict cycle which is contained in the frame $Y(u_2)$, and we reach a contradiction to $\cG_\fr$ being a framified constraint graph.

Otherwise, consider the restriction $p'$ of $p$ to the vertices of the logical elements $u_k$ and $u_{k-1}$. 
Note that we consider two logical elements, as their induced subgraph will be captured as $Y(u_k)$ in the tree representation of $\cG_\fr$.
Let $(u_{k-1},x)$ be the first vertex on $p'$ and let $(u_{k-1},y)$ be the last vertex on $p'$. 
We show that there exists some edge $e$ from $(u_{k-1},x)$ to $(u_{k-1},y)$. Due to $\cG_\fr$ being a framified constraint graph, it is enough to show that there is no strict edge from $(u_{k-1},y)$ to $(u_{k-1},x)$. Since $p'$ connects $(u_{k-1},x)$ to $(u_{k-1},y)$, if there were such a strict edge,
there would be a strict cycle in $Y(u_k)$ and we'd reach a contradiction to $\cG_\fr$ being framified. 
For the same reason, if $p'$ has a strict edge, then $e$ is strict. 

By replacing the subpath $p'$ with $e$ in $p$, we obtain a strict cycle which spans vertices of $k-1$ logical elements. 
Observe that the strictness is preserved since the potential removal of a strict edge in $p'$ is recovered by $e$ being strict.

Applying this claim inductively, we conclude that there is a strict cycle spanning two logical elements, and reach a contradiction as above.

\subsubsection*{Proof of Lemma~\ref{lem_condition_necessary}}

Let $(u,x),(v,y) \in \Delta \times \sfReg_{\cC,\cT}$.
From the definition of embeddability it immediately follows that if there is a finite
path from $(u,x)$ to $(v,y)$ of strict length $m$, then any assignment $\kappa: \Delta \times \sfReg_{\cC,\cT} \rightarrow \bZ$
witnessing the embeddability of $\cG_\fr$ would satisfy $\kappa((v,y)) - \kappa((u,x)) \geq m$.

Let $\cG_\fr$ be a framified constraint graph which does not satisfy $(\condsym)$, and let $(u,x)$ and $(u,y)$ be the violating pair.
We show that for any natural number $m$, there is a path $p$ from $(u,x)$ to $(u,y)$ 
of strict length at least $m$. Fix some $m$ and assume w.l.o.g that the forward path $f$ is strict.
Then there is a finite prefix $p_f$ of $f$ containing at least $m$ strict edges. Denote the length of $p_f$ by $l$ and let $p_b$
be the $l$-prefix of $b$.
Then the concatenation $p$ of $p_f$ with $p_b$ is a path from $(u,x)$ to $(u,y)$, since there is an edge from $f(l)$ to
$b(l)$, and $p$ is of strict length at least $m$.

%----------------------------------

\subsection*{Proof of Lemma~\ref{lem_condition_suff_for_regular}}

	This is the main technical result of the paper. 
We need some definitions and lemmas first.

\begin{definition}
\label{def_distance}
Let $(u,x),(v,z) \in \Delta \times \sfReg_{\cC,\cT}$ such that there is a path from $(u,x)$ to $(v,z)$.
If there is a finite bound on the strict length of paths from $(u,x)$ to $(v,z)$, let $m$
be the maximal strict length of such paths. Then we say 
the \emph{distance} between $(u,x)$ and $(v,z)$ is $m$. If there is no finite bound on such paths, we say 
the distance is \emph{unbounded}.
\end{definition}

\begin{lemma}
\label{lem_unbounded_distance}
Let $\cG_\fr$ be a framified constraint graph which is not embeddable into $\bZ$. Then there exist $(u,x),(v,z) \in \Delta \times \sfReg_{\cC,\cT}$
such that the distance between $(u,x)$ and $(v,z)$ is unbounded.
\end{lemma}
\begin{proof}
This is a restatement of a proposition in~\cite{CarapelleKL13} showing that $\cZ_c$ has the EHD-property. 
The defining formulas (applied to our setting) essentially state that there are no strict cycles 
(which in our case is given by the framification and Lemma~\ref{lem_no_cycles_in_framification}),
and that there exists a bound on the strict length of paths from $(u,x)$ to $(v,z)$, for every $(u,x)$ and $(v,z)$ such that $(v,z)$ is reachable from $(u,x)$. We emphasize that the bound is not global but may vary from pair to pair.
\end{proof}

In the sequel, we freely move from a constraint graph to its tree representation when discussing paths and subtrees for ease of understanding.

\begin{definition}
For $(w,x),(w,y) \in \Delta \times \sfReg_{\cC,\cT}$, we define a partial labeling 
$\ell : \Delta \times \sfReg_{\cC,\cT} \times \sfReg_{\cC,\cT} \rightarrow \mathbb{N} \cup \{\infty\}$ 
where 
\begin{enumerate}
	\item 
	if 
	the largest strict length of a simple path from $(w,x)$ to $(w,y)$ only in the subtree rooted at $w$ is $d \in \mathbb{N}$, 
	then
	$\ell((w,x),(w,y)) = d$,
	\item
	if 
	there is no bound on the strict length of a cycle-free path from $(w,x)$ to $(w,y)$ in the subtree rooted at $w$, then 
	$\ell((w,x),(w,y)) = \infty$, and
	\item
	if there is no path from $(w,x)$ to $(w,y)$ in the subtree rooted at $w$, then the label
	$\ell((w,x),(w,y))$ is not defined.
\end{enumerate}
\end{definition}
Note that the labeling only takes into account paths between vertices associated with the same logical object, and only paths in the subtree rooted at that element.
This is in contrast to Definition~\ref{def_distance}, which takes into account all paths.
We make some observations about this labeling.

\begin{lemma}
\label{lem_not_emb_infty}
Let $\mathcal{G}_\fr$ be a framified constraint graph which is not embeddable. Then there exist $u \in \Delta$ and $x,y \in \mathsf{Reg}$
with $\ell((u,x),(u,y)) = \infty$.
\end{lemma}
\begin{proof}
We first show we can restrict our attention to a single node $u$, then we show the labeling part of the lemma.
$\mathcal{G}_\fr$ is not embeddable, therefore by Lemma~\ref{lem_unbounded_distance} there exist $(u',x')$ and $(w',y')$ with unbounded distance. 
As the tree representation $\cG_\fr$ has bounded degree, by K\"onig's Lemma we have 
that there is at least one a subtree in the graph containing infinitely many subpaths of paths from $(u',x')$ to $(w',y')$ of 
infinitely many strict lengths. Let $u$ be the root of such a subtree such that $|u|$ is minimal in the sense that the previous statement holds for $u$ and does not hold for its parent (if $u$ is not $\varepsilon$). 
Since we have a bounded number of registers, again by K\"onig's Lemma
we have that there are registers $x,y$ such that the distance between $(u,x)$ and $(u,y)$ is unbounded.
By the minimality of $|u|$ we get that $\ell((u,x),(u,y)) = \infty$.

\end{proof}

\begin{definition}
Let $T$ be a regular tree over $\Sigma$. We say $w \in \Sigma^\star$ is \emph{in the repetitive part} of $T$ if
there is a prefix $u$ of $w$ such that $\left. T \right|_w = \left. T \right|_u$. 
\end{definition}

\begin{observation}
\label{obs_repetitive_part}
If $T$ is a regular tree, then any 
$w \in \Sigma^\star$ of length $|w| > \{\left. T \right|_{u} \mid u \in \Sigma^\star\}$ is in the repetitive part of $T$.
\end{observation}

\begin{lemma}
Let $\mathcal{G}_\fr$ be a regular framified constraint graph which is not embeddable.
Then there are $(w,x), (w,y) \in \Delta \times \mathsf{Reg}$ in the repetitive part such that
$\ell((w,x),(w,y)) = \infty$. 
\end{lemma}
\begin{proof}
We know from Lemma~\ref{lem_not_emb_infty} that there are some $(u,z_1), (u,z_2) \in \mathsf{Reg}$ such that $\ell((u,z_1),(u,z_2)) = \infty$.
By definition of $\ell$ and the fact we have finite degree, by K\"onig's Lemma we have that $u$ has a child $u'$ and there exist registers $z_1',z_2'$
such that $\ell((u',z_1'),(u',z_2')) = \infty$. 
We apply this argument inductively until we reach the repetitive part, which by Observation~\ref{obs_repetitive_part} is a finite number of times.
\end{proof}

For our proof it will be enough to consider partial framifications of constraint graphs.
Observe that due to the inability of framifications to introduce strict cycles (Lemma~\ref{lem_no_cycles_in_framification}),
all the framifications of a constraint graph $\cG$ contain a common subgraph whose edges relate to $\ell$ in the following way:

\begin{observation}
\label{obs_common_framification}
Let $\mathcal{G}_{\mathrm{fr}}$ be a framification
of $\mathcal{G}$. Then for every $u \in \Delta$ and $x,y \in \mathsf{Reg}$, we have in $\mathcal{G}_{\mathrm{fr}}$:
\begin{enumerate}
	\item 
	An equality edge $e_{=}((u,x),(u,y))$ if $\ell((u,x),(u,y)) = 0$
	\item
	A strict edge $e_{<}((u,x),(u,y))$ if $\ell((u,x),(u,y)) \in \mathbb{N}^+ \cup \{ \infty \}$
\end{enumerate}
\end{observation}
Note that the \emph{maximal} common subgraph may contain additional edges, as $\ell$ only takes into account paths in the subtree rooted at some vertex, but these will suffice for our proofs.

\begin{definition}
\label{def_downward_upward_trend}
We say a path $p$ has a \emph{downward trend} if the elements $w \in \Delta$ along $p$ have (strictly) increasing length. Similarly,
a path has an \emph{upward trend} if the elements have decreasing length.

\label{def_down_then_up}
Let $(w,x),(w,y) \in \Delta \times \mathsf{Reg}$. 
We say a path from $(w,x)$ to $(w,y)$ goes \emph{down-then-up} if it can be broken into two contiguous subpaths where the first subpath 
has a downward trend and the second one has an upward trend.
\end{definition}

\begin{lemma}
\label{lem_path_shape}
Let $(u,x),(u,y) \in \Delta \times \mathsf{Reg}$ such that $\ell((u,x),(u,y)) = \infty$.
Then for every $n \in \mathbb{N}$ there is a down-then-up path $p'$ from $(u,x)$ to $(u,y)$
in $\mathcal{G}_\mathrm{fr}$ of strict length at least $n$. 
\end{lemma}

\subsubsection*{Proof of Lemma~\ref{lem_path_shape}}
There are two parts to the proof. First we describe, given a path  $p$ from $(u,x)$ to $(u,y)$,
another path $p'$ from  $(u,x)$ to $(u,y)$ which goes down-then-up.
In the second part, we give a lower bound on the strict length of the new path $p'$ 
given the strict length of the original $p$.
Denote by $d$ the maximal depth of the original path $p$.

\paragraph{Constructing $p'$}
Let $u \in \Delta$ and $x,y \in \mathsf{Reg}$ such that $\ell((u,x),(u,y)) = \infty$ and let $p$ 
be a path in $\mathcal{G}_\mathrm{fr}$ from $(u,x)$ to $(u,y)$ of strict length $\geq N$ and assume $p$ has no cycles.

\begin{observation}
\label{obs_strict_subpath}
We may assume that any subpath $p''$ of $p$ which begins and ends with the same element has strict length at least $1$, 
otherwise due to framification we have an equality edge $e''$ between the start and end of $p''$. 
Then we may consider the path where $p''$ is replaced by $e''$, which has the same strict 
length as the original path $p$.
\end{observation}

\begin{observation}
The number of times we may see a certain $w \in \Delta$ along $p$ is bounded by the number of registers $|\sfReg_{\cC,\cT}|$, 
since we assume no cycles.
\end{observation}

Assume that $u$ appears exactly twice along $p$.
We inductively construct $\mathrm{pre}^{(i)}$, $\mathrm{mid}^{(i)}$, and $\mathrm{suf}^{(i)}$,
where $\mathrm{pre}^{(i)}$ has a downward trend, $\mathrm{suf}^{(i)}$ has an upward trend, and
$\mathrm{mid}^{(i)}$ remains to be altered.
For a path $q$ with endpoints $a,b$, we denote by $q \setminus \{ a, b \}$ the subpath of $q$ obtained by excluding $a$ and $b$.

Set $\mathrm{pre}^{(0)} = (u,x)$, $\mathrm{suf}^{(0)} = (u,y)$, and $\mathrm{mid}^{(0)} = p \setminus{(u,x),(u,y)}$.
Note that $\mathrm{pre}^{(0)}$ and $\mathrm{suf}^{(0)}$ have a strict downward and upward trend, respectively, and 
that $\mathrm{mid}^{(0)}$ begins and ends with vertices associated with the same logical element.

Given $\mathrm{pre}^{(i)}$, $\mathrm{suf}^{(i)}$, and $\mathrm{mid}^{(i)}$,
we define $\mathrm{pre}^{(i+1)}$, $\mathrm{suf}^{(i+1)}$, and $\mathrm{mid}^{(i+1)}$.
Denote the node appearing in the first and last vertices on $\mathrm{mid}^{(i)}$ by $u_i$.
\begin{enumerate}
	\item 
	If $u_i$ appears exactly twice on $\mathrm{mid}^{(i)}$, 
	denote its appearances by $(u_i,x_i)$ and $(u_i,y_i)$.
	Then define
	\[
	\begin{array}{llll}
	&\mathrm{pre}^{(i+1)} &=& \mathrm{pre}^{(i)} (u_i,x_i), \\ 
	&\mathrm{suf}^{(i+1)} &=& (u_i,y_i) \mathrm{suf}^{(i)}, \\ 
	\text{and }& \mathrm{mid}^{(i+1)} &=& \mathrm{mid}^{(i)} \setminus \{ (u_i,x_i) , (u_i,y_i) \}.		
	\end{array}
	\]
	\item
	If $u_i$ appears more than twice on $\mathrm{mid}^{(i)}$, 
	let $(u_i,x_i)$ and $(u_i,y_i)$ be the pair of subsequent appearances of $u_i$
	on $\mathrm{mid}^{(i)}$ whose subpath has largest strict length 
	(if there are multiple such pairs, take the earliest one).
	\begin{enumerate}
		\item 
		\label{step_add}
		Define
		\[
		\begin{array}{lll}
		\mathrm{pre}^{(i+1)} &=& \mathrm{pre}^{(i)} (u_i,x_i), \\ 
		\mathrm{suf}^{(i+1)} &=& (u_i,y_i) \mathrm{suf}^{(i)}. \\ 
		\end{array}
		\]
		Note that due to framification, there is an edge from the end of $\mathrm{pre}^{(i)}$
		to $(u_i,x_i)$, and an edge from $(u_i,y_i)$ to the beginning of $\mathrm{suf}^{(i)}$ and therefore
		$\mathrm{pre}^{(i+1)}$ and $\mathrm{suf}^{(i+1)}$ are well defined.
		Furthermore, at least one these edges is strict since at least one of them is due to Observation~\ref{obs_strict_subpath}.
		\item
		\label{step_replace}
		Let $q_{(u_i,x_i)}$ be the subpath of $\mathrm{mid}^{(i)}$ beginning at the first vertex of $\mathrm{mid}^{(i)}$ and ending at $(u_i,x_i)$.
		Let $q'_{(u_i,y_i)}$ be the subpath of $\mathrm{mid}^{(i)}$ beginning at $(u_i,y_i)$ and ending at the last vertex of $\mathrm{mid}^{(i)}$.
		Then define
		\[
		\mathrm{mid}^{(i+1)} = \mathrm{mid}^{(i)} \setminus \{ q_{(u_i,x_i)} , q'_{(u_i,y_i)} \}.
		\]
	\end{enumerate}
	\item
	If $\mathrm{mid}^{(i)}$ is empty, then define
	\[
	\begin{array}{llll}
	&\mathrm{pre}^{(i+1)} &=& \mathrm{pre}^{(i)}\\ 
	&\mathrm{suf}^{(i+1)} &=& \mathrm{suf}^{(i)} \\ 
	\text{and }& \mathrm{mid}^{(i+1)} &=& \mathrm{mid}^{(i)}.		
	\end{array}
	\]
\end{enumerate}

\begin{observation}
\label{obs_strict_edge}
If $\mathrm{mid}^{(i)}$ has strict length $> 1$, then due to framification, there is a strict edge from $(u_i,x_i)$ to $(u_i,y_i)$.
\end{observation}

After at most $d$ steps of the above construction, we will have $\mathrm{mid}^{(d)} = \epsilon$. 
Take 
$p' = \mathrm{pre}^{(d)}\mathrm{suf}^{(d)}$. 
We have that for every depth, an element of that depth appears at most twice (in fact, exactly twice except for possibly the deepest element).

\paragraph{A lower bound on the strict length of $p'$}

Note that only applications of case $2$ decrease the strict length of $p'$. 
Therefore the strict length of $p'$ will be the 
smallest when its construction involves the most applications of case $2$.
We want to bound the number of times case $2$ can be applied before we reach
$\mathrm{mid}^{(i)} = \epsilon$.

Recall that $n$ is the degree of the tree representation of $\cG_\fr$.
For $0 \leq i \leq d$, denote by $N_{\mathrm{mid}}^{(i)}$ the strict length of $\mathrm{mid}^{(i)}$.
Denote by $\rho$ the number of register names used, i.e.\ $|\sfReg_{\cC,\cT}|$.

Assume that we apply case $2$ in step $i$,
meaning $u_i$ appears more than twice on $\mathrm{mid}^{(i)}$.
Due to the degree being $n$, this implies that there is a pair $(u_i,x_i)$ and $(u_i,y_i)$ of subsequent appearances 
whose subpath has strict length at least $(N_{\mathrm{mid}}^{(i)}-\rho)/n$.
The subtraction of $\rho$ is in order to account for possibly losing $\rho$ strict edges within the same depth as we perform Step~\ref{step_replace}.
In other words,
\[
N_{\mathrm{mid}}^{(i+1)} \geq (N_{\mathrm{mid}}^{(i)}-\rho)/n
\]

Let us define this bound of $N_{\mathrm{mid}}^{(i)}$ from below as a series.
\begin{flalign*}
a_0 & =  N \\	
a_1 & =  \frac{1}{n}(a_0 - \rho) \\	
a_{i+1} & =  \frac{1}{n}(a_i - \rho)
\end{flalign*}

\begin{claim}
\label{cl_series}
For $m \geq 1$, 
\[
a_m = \frac{N-\rho}{n^m}-\rho \sum_{h=1}^{m-1}{n^{-h}}
\]
\end{claim}
\begin{proof}
By induction on $m$.
\begin{itemize}
	\item We show the claim holds for $m=1$:
	By definition, we have
	\[
	a_1 = \frac{1}{n}(a_0 - \rho) = \frac{N-\rho}{n}
	\]
	By substituting $1$ for $m$, we have:
	\[
	\left.\frac{N-\rho}{n^m} - \rho \sum_{h=1}^{m-1}{n^{-h}}\right|_{m=1} = \frac{N-\rho}{n}
	\]
	\item
	We assume the claim holds for $m=m'$:
	\[
	a_{m'} = \frac{N-\rho}{n^{m'}} - \rho \sum_{h=1}^{m'-1}{n^{-h}}
	\]
	\item
	We show correctness for $m= m'+1$:
	\begin{align*}	
	a_{m'+1} & = \frac{a_{m'}-\rho}{n} &\\
	& =  	\frac{N-\rho}{n^{m'+1}} - \rho \sum_{h=1}^{m'-1}{n^{-h-1}} - \frac{\rho}{n} &\\
	& =  \frac{N-\rho}{n^{m'+1}} - \rho \sum_{h=2}^{(m'+1)-1}{n^{-h}} - \frac{\rho}{n} &\\
	& =  \frac{N - \rho}{n^{m'+1}} - \rho \sum_{h=1}^{(m'+1)-1}{n^{-h}}
	\end{align*}
\end{itemize}
\end{proof}

Since this series bounds $N_{\mathrm{mid}}^{(m)}$ from below, we have that 
\[
N_{\mathrm{mid}}^{(m)} \geq \frac{N-\rho}{n^{m}} - \rho \sum_{h=1}^{m-1}{n^{-h}}
\]
Furthermore, we have that
\[
\frac{N-\rho}{n^{m}} - \rho \sum_{h=1}^{m-1}{n^{-h}} > \frac{N-\rho}{n^{m}} - \rho \sum_{h=1}^{\infty}{n^{-h}}
 = \frac{N-\rho}{n^{m}} - \rho \frac{1}{n-1}
\]

We solve for $m$ in order to bound the maximal number of times case $2$ may be applied in the construction of $p'$.
\[
\frac{N-\rho}{n^{m}}-\rho\frac{1}{n-1} = 0
\]
After some algebra we get
\[
m = \log_n(\frac{(N-\rho)(n-1)}{\rho})
\]

To recap -- the constructed path $p'$ has the smallest strict length if case $2$ was applied a maximal number of times, and we have showed
that this may occur at most  $\log_n(\frac{(N-\rho)(n-1)}{\rho})$ times.
However, since each time we apply case $2$ we have at least one strict edge added to the path (in Step~\ref{step_add}), 
we also have at least $\log_n(\frac{(N-\rho)(n-1)}{\rho})$ strict edges in $p'$.
Obviously, $\lim_{N \to \infty}(\log_n(\frac{(N-\rho)(n-1)}{\rho})) = \infty$, and we have our proof of Lemma~\ref{lem_path_shape}.

\paragraph{Back to the proof of Lemma~\ref{lem_condition_suff_for_regular}}
Finally, we can prove the lemma which will facilitate the construction of the paths violating $(\condsym)$.
Let $N_{\mathrm{st}}$ be the number of different subtrees in the tree representation of $\mathcal{G}_{\mathrm{fr}}$.
Let $N_{\mathrm{rep}} = \rho^4 N_{\mathrm{st}}+1$.
\begin{lemma}
Let $u \in \Delta$ be in the repetitive part of $\mathcal{G}_\fr$ and let $x,y \in \mathsf{Reg}$ such that $\ell((u,x),(u,y)) = \infty$.
Let $p$ be the path promised by Lemma~\ref{lem_path_shape} of strict length $2N_{\mathrm{rep}}$. 
Then there is $v \in \Delta$ on $p$ and registers $z,z'$ such that $v.z,v.z'$ violate $(\condsym)$. 
\end{lemma}
\begin{proof}
Since $p$ is an down-then-up path of strict length $2N_{\mathrm{rep}}$,
there are at least $N_{\mathrm{rep}}$ strict edges in one of the directions on $p$. 
Assume w.l.o.g.\ that it is the downward direction.

Since the number of strict edges in the downward direction is larger than the number of possible combinations of a subtree with a quadruple of registers, we have the following on $p$ (see Figure~\ref{fig_proof_help}):
\begin{enumerate}
	\item 
	nodes $v$ and $v' = vw$ for some $w \in [n]^+$ such that $v$ and $v'$ have isomorphic subtrees,
	\item
	registers $z,z'$ such that there is a path $f_1$ from $(v,z)$ to $(v',z)$ with strict length at least $1$, 
	and a path $b_1$ from
	$(v',z')$ to $(v,z')$
\end{enumerate}
Since $p$ is a path from $(u,x)$ to $(u,y)$ which goes through $(v',z)$ and then $(v',z')$, and since framifications may
not introduce strict cycles (Lemma~\ref{lem_no_cycles_in_framification}), we have an edge $e'$ from $(v',z)$ to  $(v',z')$. Since $v$ and $v'$ have isomorphic subtrees, this implies there is also an 
(isomorphic) edge $e$ from $(v,z)$ to $(v,z')$.
Finally, since $f_1$ is strict, by Observation~\ref{obs_common_framification} we have that the edges $e$ and $e'$ are strict.  

As $v$ and $v'$ have isomorphic subtrees, 
this implies that there is an infinite strict forward path $f$ from $(v,z)$, since the strict forward path $f_1$ can be concatenated indefinitely.
Similarly, there is an infinite backward path $b$ into $(v,z')$, as the path $b_1$ can also be concatenated.

It remains to show that there is strict edge from $f(i)$ to $b(i)$ for every $i \geq 0$.
By our construction, we have that 
for every $i$, there is a strict path from $f(i)$ to $b(i)$ -- for example one which uses a copy of the edge $e$ above. 
Furthermore, by the construction in the proof of Lemma~\ref{lem_path_shape}, for every $i$, $f(i)$ and $b(i)$ are vertices 
associated with the same logical element.
Therefore by framification we have a strict edge from $f(i)$ to $b(i)$.

\begin{figure}
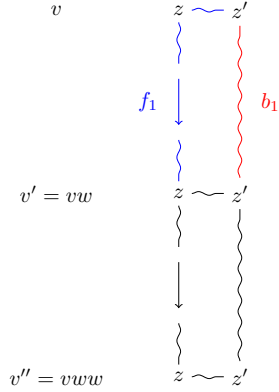
%
\centering
\resizebox{4cm}{!}{
\pathpic
}
\caption{Subgraph of the constraint graph. $v$, $v'$, $v''$ have isomorphic subtrees along a periodic word $w$, $f_1$ is a forward back with at least one strict edge and $b_1$ is a backward path. Note that $f_1$ goes from the $z$ register of a vertex to the $z$ register of an isomorphic vertex (and $b_1$ behaves similarly).}%
\label{fig_proof_help}%
\end{figure}

\end{proof}

\section*{The automata in Subsections~\ref{subsec_automata_const} and~\ref{subsec_putting_automata_together}}
Here we give the construction of the automata $\cA_{\mathrm{ct}}$, $\cA_{\mathrm{alcf}}$, and $\cA_{\mathrm{m}}$.

\paragraph{Checking consistency of trees}
First we denote the pairs of consistent pairs of frames as in Definition~\ref{def_consistent_frames}: 
$\mathrm{CFr} = \{ (\sigma_1, \sigma_2) \in \Sigma_\fr \times \Sigma_\fr \mid (\sigma_1,\sigma_2) \text{ is consistent}\}$.
Also denote the set of frames whose vertex set only has $\rbot$ vertices: $\Sigma_\fr^{\rbot} = \{\sigma \in \Sigma_\fr \mid V(\sigma) = \sfReg_{\cC,\cT}^{\rbot} \}$.

We define $\cA_{\mathrm{ct}} = (Q, q_0, \longrightarrow, (\emptyset, Q))$, where:
\begin{itemize}
	\item 
	$Q = \{q_0\} \cup \{q_{\sigma} \mid \sigma \in \Sigma_\fr\}$.
	\item
	For every $\sigma \in \Sigma_\fr^{\rbot}$, we have $q_0 \overset{\sigma}{\longrightarrow} (q_\sigma, \ldots, q_\sigma)$.
	\item
	For every $(\sigma_1, \sigma_2) \in \rCFr$, we have $q_{\sigma_1} \overset{\sigma_2}{\longrightarrow} (q_{\sigma_2}, \ldots, q_{\sigma_2})$.
\end{itemize}

\paragraph{Satisfiability of the abstracted $\ALCF$ part}
The construction is very nearly identical to the one in~\cite{Baader09}, therefore we only describe the parts needed to understand our adaptation.

Following the notation in~\cite{Baader09}, let $S_{\cC,\cT}$ be the set of subexpressions of $\cC$ and $\cT$, and let
$R_{\cC,\cT}$ be the set of role names used in $\cC$ and $\cT$. 
The state set of $\cA_\ralcf$ contains the Hintikka sets for $\cC,\cT$, i.e.\ $q \subseteq S_{\cC,\cT} \cup R_{\cC,\cT}$ where either $q = \emptyset$
or $q$ contains exactly one role name and maximally consistent subexpressions. 

Our automaton $\cA_\ralcf$ has $\Omega = (\emptyset, Q)$ where $Q$ is the entire state set, and its transition relation only
differs from the one in~\cite{Baader09} in order to properly handle functional roles. Specifically, 
our $(q,\xi) \longrightarrow (q_1, \ldots, q_n)$ additionally satisfies that
\begin{itemize}
	\item 
	if $\exists r. D \in q$ for $r \in \sfNR$, then there is exactly one $i$ such that $\{D,r\} \subseteq q_i$
	\item
	if $\forall r. D \in q$ for $r \in \sfNR$, then either: 
	\begin{itemize}
		\item 
		there is no $i$ such that $\{r\} \subseteq q_i$, or
		\item
		$n=1$ and $q \overset{\xi}{\longrightarrow} (q_1)$ where $\{D,r\} \subseteq q_1$
	\end{itemize}
\end{itemize}

\paragraph{Matching the alphabets}
We need to verify that the graph induced by the $\Xi$ part of the letter is contained the in framification 
of the $\sfReg_{\cC,\cT}^\rbot$ vertices in the $\Sigma_\fr$ part. We introduce some notation. 
Recall that $\bfB$ is the set of placeholders introduced during the abstraction of $\cC,\cT$.
For $\xi \in \Xi$, denote $\bfB(\xi) = \xi \cap \bfB$, i.e.\ the set of placeholders appearing in $\xi$.
For $\sigma \in \Sigma_\fr$ denote 
\[
\bfB_{\rbot}(\sigma) = \{ B \in \bfB \mid \text{ there is } v \in \sfReg_{\cC,\cT}^\rbot \text{ in } V(\sigma) \text{ s.t } B \in \lambda(\sigma)\}
\]
I.e.\ the placeholders appearing on $\sfReg_{\cC,\cT}^\rbot$ vertices in $\sigma$.
Now we define $\cA_{\mathrm{m}}$ to simply ensure we always have $\bfB(\xi) \subseteq \bfB_{\rbot}(\sigma)$. More precisely,
we define $\cA_{\mathrm{m}} = (Q, q_0,\longrightarrow,(\emptyset, Q))$ where
\begin{itemize}
	\item 
	$Q = \{q_0\}$
	\item
	For every $(\sigma, \xi) \in \Sigma_\fr \times \Xi$ where $\bfB(\xi) \subseteq \bfB_{\rbot}(\sigma)$, we have 
	$q_0 \overset{(\sigma,\xi)}{\longrightarrow} (q_0, \ldots, q_0)$
\end{itemize}

\paragraph{Product of Rabin tree automata}

The product of two Rabin tree automata is obtained by simply taking the product of the state sets, transition relation, and accepting pairs as follows.
Let $\cA = (Q, q_0, \longrightarrow, \{(L_1,U_1), \ldots, (L_m,U_m)\})$ 
and $\cA' = (Q', q_0', \longrightarrow', \{(L_1',U_1'), \ldots, (L'_{m'},U'_{m'})\})$ be two Rabin tree automata over some alphabet $\Gamma$ which run on $n$-trees.
The automaton $\cA^{\cap} = (Q^\cap, q_0^\cap, \longrightarrow^\cap, \Omega^\cap)$ is given by
\begin{itemize}
	\item 
	$Q^\cap = Q \times Q'$
	\item
	$q_0^\cap = (q_0, q_0')$
	\item
	$(q,q') {\longrightarrow^\cap} ((q_1,q_1'), \ldots, (q_n,q_n'))$ with the letter $\gamma$ if and only if 
	$q {\longrightarrow} (q_1, \ldots, q_n)$ with $\gamma$ and 
	$q' {\longrightarrow'} (q_1', \ldots, q_n')$ with $\gamma$.
	\item
	$\Omega^\cap = \{((L_i,L'_j),(U_i,U_j')) \mid i \in [m], j \in [m']\}$
\end{itemize}
That is, $\cA^\cap$ runs $\cA$ and $\cA'$ simultaneously. 
For a run on some input, we have acceptance by both $\cA$ and $\cA'$ if and only if we have that every path has some $i$ such that its restriction to $Q$ is successful due to $(L_i,U_i)$, and some $j$ such that its restriction to $Q'$ is successful due to $(L_j',U_j')$. Therefore, a given input is accepted by both $\cA$ and $\cA'$ if and only if there is a run where every path has some $(i,j)$ witnessing its success, i.e. there is an accepting run of $\cA^\cap$.

\end{document}